\documentclass{article}

\usepackage[numbers]{natbib}
     \usepackage[final]{neurips_2019}

\usepackage[utf8]{inputenc} % allow utf-8 input
\usepackage[T1]{fontenc}    % use 8-bit T1 fonts
\usepackage{hyperref}       % hyperlinks
\usepackage{url}            % simple URL typesetting
\usepackage{booktabs}       % professional-quality tables
\usepackage{amsfonts}       % blackboard math symbols
\usepackage{nicefrac}       % compact symbols for 1/2, etc.
\usepackage{microtype}      % microtypography

\usepackage{float}
\usepackage{calc}
\usepackage{amsmath}
\usepackage{amsthm}

\providecommand{\tabularnewline}{\\}
\floatstyle{ruled}
\newfloat{algorithm}{tbp}{loa}
\providecommand{\algorithmname}{Algorithm}
\floatname{algorithm}{\protect\algorithmname}

\theoremstyle{plain}
\newtheorem{thm}{\protect\theoremname}
\theoremstyle{remark}
\newtheorem{rem}[thm]{\protect\remarkname}
\theoremstyle{plain}
\newtheorem{cor}[thm]{\protect\corollaryname}
\theoremstyle{plain}
\newtheorem{prop}[thm]{\protect\propositionname}
\theoremstyle{definition}
\newtheorem{defn}[thm]{\protect\definitionname}
\theoremstyle{plain}
\newtheorem{lem}[thm]{\protect\lemmaname}
\theoremstyle{plain}
\newtheorem{fact}[thm]{\protect\factname}
\theoremstyle{definition}
\newtheorem{example}[thm]{\protect\examplename}
\usepackage{algorithm,algpseudocode,bbm,dsfont, amsfonts, xspace, enumitem, graphicx, subfig}

\def\R{\mathbb{R}}

\def\E{\mathbb{E}}

\def\One{\mathds{1}}
\def\calX{\mathcal{X}}
\def\calY{\mathcal{Y}}
\def\calH{\mathcal{H}}

\def\d{\mathop{}\!\mathrm{d}}

\newenvironment{subassu}
    {\par\medskip\noindent\refstepcounter{subassu}%
     \hbox{\it Assumption \thesubassu.}\it\ignorespaces}
    {\medskip}

\def\hvar{\hat{\text{V}}}

\def\var{\text{V}}

\def\vs{C}
\def\pasdat{T_{0}}
\def\Bin{\text{Bin}}
\def\VCIS{Variance-Controlled Importance Sampling\xspace}
\def\vcis{variance-controlled importance sampling\xspace}
\def\DEBIAS{Sample Selection Bias Correction\xspace}
\def\debias{sample selection bias correction\xspace}
\def\algoPassive{\textsc{Passive}\xspace}
\def\algoActive{\textsc{Active18}\xspace}
\def\algoVCActive{\textsc{ActiveVC}\xspace}
\setlist[itemize]{leftmargin=*}
\makeatother

\providecommand{\corollaryname}{Corollary}
\providecommand{\definitionname}{Definition}
\providecommand{\examplename}{Example}
\providecommand{\factname}{Fact}
\providecommand{\lemmaname}{Lemma}
\providecommand{\propositionname}{Proposition}
\providecommand{\remarkname}{Remark}
\providecommand{\theoremname}{Theorem}

\title{The Label Complexity of Active Learning from Observational Data}

\author{%
  Songbai Yan \\
  University of California San Diego\\
  \texttt{yansongbai@eng.ucsd.edu}\\
  \And
  Kamalika Chaudhuri \\
  University of California San Diego \\
  \texttt{kamalika@cs.ucsd.edu} \\
  \And
  Tara Javidi \\
  University of California San Diego \\
  \texttt{tjavidi@eng.ucsd.edu} \\  
}

\begin{document}

\maketitle

%\begin{abstract}
%Counterfactual learning from observational data is an emerging problem that arises naturally in many applications. This work considers counterfactual learning in the active setting, where in addition to observational data, the learner has access to unlabeled examples from an underlying population and the ability to label a subset of those. The learner's goal is to produce a classifier that has good performance on the entire population, and not just the observational data.

%Prior work on this problem uses disagreement-based active learning, along with an importance weighted loss estimator to account for counterfactuals. This work shows how to instead incorporate a more efficient counterfactual risk minimizer into active learning. This requires us to modify both the counterfactual risk to make it amenable to active learning, as well as the active learning process to make it amenable to the risk. We provably demonstrate that the result of this is an algorithm which is statistically consistent as well as more sample-efficient than prior work. 
%\end{abstract}

\begin{abstract}
Counterfactual learning from observational data involves learning
a classifier on an entire population based on data that is observed
conditioned on a selection policy. This work considers this problem
in an active setting, where the learner additionally has access to
unlabeled examples and can choose to get a subset of these labeled
by an oracle. 

Prior work on this problem uses disagreement-based active learning,
along with an importance weighted loss estimator to account for counterfactuals,
which leads to a high label complexity. We show how to instead incorporate
a more efficient counterfactual risk minimizer into the active learning
algorithm. This requires us to modify both the counterfactual risk
to make it amenable to active learning, as well as the active learning
process to make it amenable to the risk. We provably demonstrate that
the result of this is an algorithm which is statistically consistent
as well as more label-efficient than prior work.
\end{abstract}

\section{Introduction}

Counterfactual learning from observational data is an emerging problem
that arises naturally in many applications. In this problem, the learner
is given observational data -- a set of examples selected according
to some policy along with their labels -- as well as access to the
policy that selects the examples, and the goal is to construct a classifier
with high performance on an entire population, not just the observational
data distribution. An example is learning to predict if a treatment
will be effective based on features of a patient. Here, we have some
observational data on how the treatment works for patients that were
assigned to it, but if the treatment is given only to a certain category
of patients, then the data is not reflective of the population. Thus
the main challenge in counterfactual learning is how to counteract
the effect of the observation policy and build a classifier that applies
more widely.

This work considers counterfactual learning in the active setting,
which has received very recent attention in a few different contexts~\cite{YCJ18,SSSVSK19,AZS18}.
In addition to observational data, the learner has an online stream
of unlabeled examples drawn from the underlying population distribution,
and the ability to selectively label a subset of these in an interactive
manner. The learner's goal is to again build a classifier while using
as few label queries as possible. The advantage of the active over
the passive is its potential for more label-efficient solutions; the
question however is how to do this algorithmically.

Prior work in this problem has looked at both probabilistic inference~\cite{SSSVSK19,AZS18}
as well as a standard classification~\cite{YCJ18}, which is the
setting of our work. \cite{YCJ18} uses a modified version of disagreement-based
active learning~\cite{CAL94,DHM07,BBL09,H14}, along with an importance
weighted empirical risk to account for the population. However, a
problem with this approach is that the importance weighted risk estimator
can have extremely high variance when the importance weights -- that
reflect the inverse of how frequently an instance in the population
is selected by the policy -- are high; this may happen if, for example,
certain patients are rarely given the treatment. This high variance
in turn results in high label requirement for the learner.

The problem of high variance in the loss estimator is addressed in
the passive case by minimizing a form of counterfactual risk~\cite{SJ15CRM}
-- an importance weighted loss that combines a variance regularizer
and importance weight clipping or truncation to achieve low generalization
error. A plausible solution is to use this risk for active learning
as well. However, this cannot be readily achieved for two reasons.
The first is that the variance regularizer itself is a function of
the entire dataset, and is therefore challenging to use in interactive
learning where data arrives sequentially. The second reason is that
the minimizer of the (expected) counterfactual risk depends on $n$,
the data size, which again is inconvenient for learning in an interactive
manner. 

In this work, we address both challenges. To address the first, we
use, instead of a variance regularizer, a novel regularizer based
on the second moment; the advantage is that it decomposes across multiple
segments of the data set as which makes it amenable for active learning.
We provide generalization bounds for this modified counterfactual
risk minimizer, and show that it has almost the same performance as
counterfactual risk minimization with a variance regularizer~\cite{SJ15CRM}.
The second challenge arises because disagreement-based active learning
ensures statistical consistency by maintaining a set of plausible
minimizers of the expected risk. This is problematic when the minimizer
of the expected risk itself changes between iterations as in the case
with our modified regularizer. We address this challenge by introducing
a novel variant of disagreement-based active learning which is always
guaranteed to maintain the population error minimizer in its plausible
set.

Additionally, to improve sample efficiency, we then propose a third
novel component -- a new sampling algorithm for correcting sample
selection bias that selectively queries labels of those examples which
are underrepresented in the observational data. Combining these three
components gives us a new algorithm. We prove this newly proposed
algorithm is statistically consistent -- in the sense that it converges
to the true minimizer of the population risk given enough data. We
also analyze its label complexity, show it is better than prior work~\cite{YCJ18},
and demonstrate the contribution of each component of the algorithm
to the label complexity bound.

\section{Related Work}

We consider learning with logged observational data where the logging
policy that selects the samples to be observed is known to the learner.
The standard approach is importance sampling to derive an unbiased
loss estimator~\cite{RR83}, but this is known to suffer from high
variance. One common approach for reducing variance is to clip or
truncate the importance weights~\cite{BPQC+13,SJ15CRM}, and we provide
a new principled method for choosing the clipping threshold with theoretical
guarantees. Another approach is to add a regularizer based on empirical
variance to the loss function to favor models with low loss variance~\cite{MP09,SJ15CRM,ND17}.
Our second moment regularizer achieves a similar effect, but has the
advantage of being applicable to active learning with theoretical
guarantees.

In this work, in addition to logged observational data, we allow the
learner to actively acquire additional labeled examples. The closest
to our work is~\cite{YCJ18}, the only known work in the same setting.
\cite{YCJ18} and our work both use disagreement-based active learning
(DBAL) framework~\cite{CAL94,DHM07,BBL09,H14} and multiple importance
sampling~\cite{VG95} for combining actively acquired examples with
logged observational data. \cite{YCJ18} uses an importance weighted
loss estimator which leads to high variance and hence high sample
complexity. In our work, we incorporate a more efficient \vcis into
active learning and show that it leads to a better label complexity.

\cite{AZS18} and \cite{SSSVSK19} consider active learning for predicting
individual treatment effect which is similar to our task. They take
a Bayesian approach which does not need to know the logging policy,
but assumes the true model is from a known distribution family. Additionally,
they do not provide label complexity bounds. A related line of research
considers active learning for domain adaptation, and their methods
are mostly based on heuristics~\cite{SRDVD11,ZJLDY16}, utilizing
a clustering structure~\cite{KGRHL15}, or non-parametric methods~\cite{KM18}.
In other related settings, \cite{ZADLN} considers warm-starting contextual
bandits targeting at minimizing the cumulative regret instead of the
final prediction error; \cite{KAHHL17} studies active learning with
bandit feedback without any logged observational data.

\section{Problem Setup}

We are given an instance space $\mathcal{X}$, a label space $\mathcal{Y}=\{-1,+1\}$,
and a hypothesis class $\mathcal{H}\subset\mathcal{Y}^{\mathcal{X}}$.
Let $D$ be an underlying data distribution over $\mathcal{X}\times\mathcal{Y}$.
For simplicity, we assume $\mathcal{H}$ is a finite set, but our
results can be generalized to VC-classes by standard arguments~\cite{VC71,ND17}.

In the passive setting for learning with observational data, the learner
has access to a logged observational dataset generated from the following
process. First, $m$ examples $\{(X_{t},Y_{t})\}_{t=1}^{m}$ are drawn
i.i.d. from $D$. Then a logging policy $Q_{0}:\mathcal{X}\rightarrow[0,1]$
that describes the probability of observing the label is applied.
In particular, for each example $(X_{t},Y_{t})$ ($1\leq t\leq m$),
an independent Bernoulli random variable $Z_{t}$ with expectation
$Q_{0}(X_{t})$ is drawn, and then the label $Y_{t}$ is revealed
to the learner if $Z_{t}=1$\footnote{This generating process implies the standard unconfoundedness assumption
in the counterfactual inference literature: $\Pr(Y_{t},Z_{t}\mid X_{t})=\Pr(Y_{t}\mid X_{t})\Pr(Z_{t}\mid X_{t})$.
In other words, the label $Y_{t}$ is conditionally independent with
the action $Z_{t}$ (indicating whether the label is observed) given
the instance $X_{t}$.}. We call $T_{0}=\{(X_{t},Y_{t},Z_{t})\}_{t=1}^{m}$ the logged dataset.
We assume the learner knows the logging policy $Q_{0}$, and only
observes instances $\{X_{t}\}_{t=1}^{m}$, indicators $\{Z_{t}\}_{t=1}^{m}$,
and revealed labels $\{Y_{t}\mid Z_{t}=1\}_{t=1}^{m}$.

In the active learning setting, in addition to the logged dataset,
the learner has access to a stream of online data. In particular,
there is a stream of additional $n$ examples $\{(X_{t},Y_{t})\}_{t=m+1}^{m+n}$
drawn i.i.d. from distribution $D$. At time $t$ ($m<t\leq m+n$),
the learner applies a query policy to compute an indicator $Z_{t}\in\{0,1\}$,
and then the label $Y_{t}$ is revealed if $Z_{t}=1$. The computation
of $Z_{t}$ may in general be randomized, and is based on the observed
logged data $T_{0}$, previously observed instances $\{X_{i}\}_{i=m+1}^{t}$,
decisions$\{Z_{i}\}_{i=m+1}^{t-1}$, and observed labels $\{Y_{i}\mid Z_{i}=1\}_{i=m+1}^{t-1}$.

We focus on the active learning setting, and the goal of the learner
is to learn a classifier $h\in\mathcal{H}$ from observed logged data
and online data. Fixing $D$, $Q_{0}$, $m$, $n$, the performance
is measured by: (1) the error rate $l(h):=\Pr_{D}(h(X)\neq Y)$ of
the output classifier, and (2) the number of label queries on the
online data. Note that the error rate is over the entire population
$D$ instead of conditioned on the logging policy, and that we assume
the labels of the logged data $T_{0}$ come at no cost. In this work,
we are interested in the situation where $n$, the size of the online
stream, is smaller than $m$.

\paragraph{Notation}

Unless otherwise specified, all probabilities and expectations are
over the draw of all random variables $\{(X_{t},Y_{t},Z_{t})\}_{t=1}^{m+n}$.
Define $q_{0}=\inf_{x}Q_{0}(x)$. Define the optimal classifier $h^{\star}=\arg\min_{h\in\calH}l(h)$,
$\nu=l(h^{\star})$. For any $r>0,h\in\calH$, define the $r-$ball
around $h$ as $B(h,r)=\left\{ h'\in\calH:\Pr(h(X)\neq h'(X))\leq r\right\} $.
For any $C\subseteq\calH$, define the disagreement region $\text{DIS}(C)=\{x\in\calX:\exists h_{1}\neq h_{2}\in C,h_{1}(X)\neq h_{2}(X)\}$.

Due to space limit, all proofs are postponed to Appendix.

\section{\label{sec:Bias-Variance-Trade-Off}\VCIS for Passive Learning with
Observational Data}

In the passive setting, the standard method to overcome sample selection
bias is to optimize the importance weighted (IW) loss $l(h,\pasdat)=\frac{1}{m}\sum_{t}\frac{\One\{h(X_{t})\neq Y_{t}\}Z_{t}}{Q_{0}(X_{t})}$.
This loss is an unbiased estimator of the population error $\Pr(h(X)\neq Y)$,
but its variance $\frac{1}{m}\E(\frac{\One\{h(X)\neq Y\}Z}{Q_{0}(X)}-l(h))^{2}$
can be high, leading to poor solutions. Previous work addresses this
issue by adding a variance regularizer~\cite{MP09,SJ15CRM,ND17}
and clipping/truncating the importance weight~\cite{BPQC+13,SJ15CRM}.
However, the variance regularizer is challenging to use in interactive
learning when data arrives sequentially, and it is unclear how the
clipping/truncating threshold should be chosen to yield good theoretical
guarantees.

%In this paper, we propose two novel techniques. First, as an alternative to the variance regularizer, we propose a second moment regularizer which achieves a similar error bound and can be applied for active learning. Second, we propose a new principled method to choosing the clipping threshold for clipped importance sampling with theoretical guarantees.

In this paper, as an alternative to the variance regularizer, we propose
a novel second moment regularizer which achieves a similar error bound
to the variance regularizer~\cite{ND17}; and this motivates a principled
choice of the clipping threshold.

\subsection{Second-Moment-Regularized Empirical Risk Minimization}

Intuitively, between two classifiers with similarly small training
loss $l(h,\pasdat)$, the one with lower variance should be preferred,
since its population error $l(h)$ would be small with a higher probability
than the one with higher variance. Existing work encourages low variance
by regularizing the loss with the estimated variance $\hat{\text{Var}}(h,\pasdat)=\frac{1}{m}\sum_{i}(\frac{\One\{h(X_{i})\neq Y_{i}\}Z_{i}}{Q_{0}(X_{i})})^{2}-l(h,\pasdat)^{2}$.
Here, we propose to regularize with the estimated second moment $\hvar(h,\pasdat)=\frac{1}{m}\sum_{i}(\frac{\One\{h(X_{i})\neq Y_{i}\}Z_{i}}{Q_{0}(X_{i})})^{2}$,
an upper bound of $\hat{\text{Var}}(h,\pasdat)$. We have the following
generalization error bound for regularized ERM.
\begin{thm}
\label{thm:paper-smr-erm-gen}Let $\hat{h}=\arg\min_{h\in\calH}l(h,\pasdat)+\sqrt{\frac{4\log\frac{|\calH|}{\delta}}{m}\hvar(h,\pasdat)}$.
For any $\delta>0$, then with probability at least $1-\delta$, $l(\hat{h})-l(h^{\star})\leq\frac{28\log\frac{|\mathcal{H}|}{\delta}}{3mq_{0}}+\sqrt{\frac{4\log\frac{|\calH|}{\delta}}{m}\E\frac{\One\{h^{\star}(X)\neq Y\}}{Q_{0}(X)}}+\frac{\sqrt{4\log\frac{|\mathcal{H}|}{\delta}}}{m^{\frac{3}{2}}q_{0}^{2}}.$
\end{thm}

Theorem~\ref{thm:paper-smr-erm-gen} shows an error rate similar
to the one for the variance regularizer \cite{ND17}. However, the
advantage of using the second moment is the decomposability: $\hvar(h,S_{1}\cup S_{2})=\frac{|S_{1}|}{|S_{1}|+|S_{2}|}\hvar(h,S_{1})+\frac{|S_{2}|}{|S_{1}|+|S_{2}|}\hvar(h,S_{2})$.
This makes it easier to analyze for active learning that we will discuss
later.

Recall for the unregularized importance sampling loss minimizer $\hat{h}_{\text{IW}}=\arg\min_{h\in\calH}l(h,\pasdat)$,
the error bound is $\tilde{O}(\frac{\log|\mathcal{H}|}{mq_{0}}+\sqrt{\frac{\log|\mathcal{H}|}{m}\min(\frac{l(h^{\star})}{q_{0}},\E\frac{1}{Q_{0}(X)})})$
\cite{CMM10,YCJ18}. In Theorem~\ref{thm:paper-smr-erm-gen}, the
extra $\frac{1}{m^{\frac{3}{2}}q_{0}^{2}}$ term is due to the deviation
of $\sqrt{\hvar(h,\pasdat)}$ around $\sqrt{\E\frac{\One\{h(X)\neq Y\}}{Q_{0}(X)}}$,
and is negligible when $m$ is large. In this case, learning with
a second moment regularizer gives a better generalization bound.

This improvement in generalization error is due to the regularizer
instead of tighter analysis. Similar to \cite{MP09,ND17}, we show
in Theorem~\ref{thm:paper-iw-gen-lb} that for some distributions,
the error bound in Theorem~\ref{thm:paper-smr-erm-gen} cannot be
achieved by any algorithm that simply optimizes the unregularized
empirical loss.
\begin{thm}
\label{thm:paper-iw-gen-lb}For any $0<\nu<\frac{1}{3}$, $m\geq\frac{49}{\nu^{2}}$,
there is a sample space $\calX\times\calY$, a hypothesis class $\calH$,
a distribution $D$, and a logging policy $Q_{0}$ such that $\frac{\nu}{q_{0}}>\E\frac{\One\{h^{\star}(X)\neq Y\}}{Q_{0}(X)}$,
and that with probability at least $\frac{1}{100}$ over the draw
of $S=\{(X_{t},Y_{t},Z_{t})\}_{t=1}^{m}$, if $\hat{h}=\arg\min_{h\in\calH}l(h,S)$,
then $l(\hat{h})\geq l(h^{\star})+\frac{1}{mq_{0}}+\sqrt{\frac{\nu}{mq_{0}}}$.
\end{thm}

\subsection{Clipped Importance Sampling}

The variance and hence the error bound for second-moment regularized
ERM can still be high if $\frac{1}{Q_{0}(x)}$ is large. This $\frac{1}{Q_{0}(X)}$
factor arises inevitably to guarantee the importance weighted estimator
is unbiased. Existing work alleviates the variance issue at the cost
of some bias by clipping or truncating the importance weight. In this
paper, we focus on clipping, where the loss estimator becomes $l(h;\pasdat,M):=\frac{1}{m}\sum_{i=1}^{m}\frac{\One\{h(X_{i})\neq Y_{i}\}Z_{i}}{Q_{0}(X_{i})}\One[\frac{1}{Q_{0}(X_{i})}\leq M]$.
This estimator is no longer unbiased, but as the weight is clipped
at $M$, so is the variance. Although studied previously \cite{BPQC+13,SJ15CRM},
to the best of our knowledge, it remains unclear how the clipping
threshold $M$ can be chosen in a principled way.

We propose to choose $M_{0}=\inf\{M'\geq1\mid\frac{2M'\log\frac{|\mathcal{H}|}{\delta}}{m}\geq\Pr_{X}(\frac{1}{Q_{0}(X)}>M')\}$.
This choice minimizes an error bound for the clipped second-moment
regularized ERM and we formally show this in Appendix~\ref{sec:app-vcis}.
Example~\ref{exa:clipping} in Appendix~\ref{sec:app-vcis} shows
this clipping threshold avoids outputting suboptimal classifiers.
The choice of $M_{0}$ implies that the clipping threshold should
be larger as the sample size $m$ increases, which confirms the intuition
that with a larger sample size the variance becomes less of an issue
than the bias. We have the following generalization error bound.
\begin{thm}
\label{cor:paper-gen-err}Let $\hat{h}=\arg\min_{h\in\calH}l(h;\pasdat,M_{0})+\sqrt{\frac{4\log\frac{|\calH|}{\delta}}{m}\hvar(h;\pasdat,M_{0})}$.
For any $\delta>0$, with probability at least $1-\delta$, 
\[
l(\hat{h})-l(h^{\star})\leq\frac{34\log\frac{|\mathcal{H}|}{\delta}}{3m}M_{0}+\frac{\sqrt{4\log\frac{|\mathcal{H}|}{\delta}}}{m^{\frac{3}{2}}}M_{0}^{2}+\sqrt{\frac{4\log\frac{|\calH|}{\delta}}{m}\E\frac{\One\{h^{\star}(X)\neq Y\}}{Q_{0}(X)}\One[\frac{1}{Q_{0}(X)}\leq M_{0}]}.
\]
\end{thm}

We always have $M_{0}\leq\frac{1}{q_{0}}$ as $\Pr_{X}(\frac{1}{Q_{0}(X)}>\frac{1}{q_{0}})=0$.
Thus, this error bound is always no worse than that without clipping
asymptotically.

\section{\label{sec:Active-Learning}Active Learning with Observational Data}

Next, we consider active learning where in addition to a logged observational
dataset the learner has access to a stream of unlabeled samples from
which it can actively query for labels. The main challenges are how
to control the variance due to the observational data with active
learning, and how to leverage the logged observational data to reduce
the number of label queries beyond simply using them for warm-start.

To address these challenges, we first propose a nontrivial change
to the Disagreement-Based Active Learning (DBAL) so that the \vcis
objective can be incorporated. This modified algorithm also works
in a general cost-sensitive active learning setting which we believe
is of independent interest. Second, we show how to combine logged
observational data with active learning through multiple importance
sampling (MIS). Finally, we propose a novel \debias technique to
query regions under-explored in the observational data more frequently.
We provide theoretical analysis demonstrating that the proposed method
gives better label complexity guarantees than previous work \cite{YCJ18}
and alternative methods.

\subsubsection*{Key Technique 1: Disagreement-Based Active Learning with \VCIS}

The DBAL framework is a widely-used general framework for active learning
\cite{CAL94,DHM07,BBL09,H14}. This framework iteratively maintains
a candidate set $\vs_{t}$ to be a confidence set for the optimal
classifier. A disagreement region $D_{t}$ is then defined accordingly
to be the set of instances on which there are two classifiers in $\vs_{t}$
that predict labels differently. At each iteration, it draws a set
of unlabeled instances. The labels for instances falling inside the
disagreement region are queried; otherwise, the labels are inferred
according to the unanimous prediction of the candidate set. These
instances with inferred or queried labels are then used to shrink
the candidate set.

The classical DBAL framework only considers the unregularized 0-1
loss. As discussed in the previous section, with observational data,
unregularized loss leads to suboptimal label complexity. However,
directly adding a regularizer breaks the statistical consistency of
DBAL, since the proof of its consistency is contingent on two properties:
(1) the minimizer of the population loss $l(h)$ stays in all candidate
sets with high probability; (2) the loss difference $l(h_{1},S)-l(h_{2},S)$
for any $h_{1},h_{2}\in\vs_{t}$ does not change no matter how examples
outside the disagreement region $D_{t}$ are labeled.

Unfortunately, if we add a variance based regularizer (either estimated
variance or second moment), the objective function $l(h,S)+\sqrt{\frac{\lambda}{n}\hvar(h,S)}$
has to change as the sample size $n$ increases, and so does the optimal
classifier w.r.t.~regularized population loss $\tilde{h}_{n}=\arg\min l(h)+\sqrt{\frac{\lambda}{n}V(h)}$.
Consequently, $\tilde{h}_{n}$ may not stay in all candidate sets.
Besides, the difference of the regularized loss $l(h_{1},S)+\sqrt{\frac{\lambda}{n}\hvar(h_{1},S)}-(l(h_{2},S)+\sqrt{\frac{\lambda}{n}\hvar(h_{2},S)})$
changes if labels of examples outside the disagreement region $D_{t}$
are modified, breaking the second property.

To resolve the consistency issues, we first carefully choose the definition
of the candidate set and guarantee the optimal classifier w.r.t.~the
prediction error $h^{\star}=\arg\min l(h)$, instead of the regularized
loss $\tilde{h}_{n}$, stays in candidate sets with high probability.
Moreover, instead of the plain variance regularizer, we apply the
second moment regularizer and exploit its decomposability property
to bound the difference of the regularized loss for ensuring consistency.

\subsubsection*{Key Technique 2: Multiple Importance Sampling}

MIS addresses how to combine logged observational data with actively
collected data for training classifiers~\cite{ABSJ17,YCJ18}. To
illustrate this, for simplicity, we assume a fixed query policy $Q_{1}$
is used for active learning. To make use of both $T_{0}=\{(X_{i},Y_{i},Z_{i})\}_{i=1}^{m}$
collected by $Q_{0}$ and $T_{1}=\{(X_{i},Y_{i},Z_{i})\}_{i=m+1}^{m+n}$
collected by $Q_{1}$, one could optimize the unbiased importance
weighted error estimator $l_{\text{IS}}(h,T_{0}\cup T_{1})=\sum_{i=1}^{m}\frac{\One\{h(X_{i})\neq Y_{i}\}Z_{i}}{(m+n)Q_{0}(X_{i})}+\sum_{i=m+1}^{m+n}\frac{\One\{h(X_{i})\neq Y_{i}\}Z_{i}}{(m+n)Q_{1}(X_{i})}$
which can have high variance and lead to poor generalization error.
Here, we apply the MIS estimator $l_{\text{MIS}}(h,T_{0}\cup T_{1}):=\sum_{i=1}^{m+n}\frac{\One\{h(X_{i})\neq Y_{i}\}Z_{i}}{mQ_{0}(X_{i})+nQ_{1}(X_{i})}$
which effectively treats the data $T_{0}\cup T_{1}$ as drawn from
a mixture policy $\frac{mQ_{0}+nQ_{1}}{m+n}$. $l_{\text{MIS}}$ is
also unbiased, but has lower variance than $l_{\text{IS}}$ and thus
gives better error bounds.

\subsubsection*{Key Technique 3: Active \DEBIAS}

Another advantage to consider active learning is that the learner
can apply a strategy to correct the sample selection bias, which improves
label efficiency further. This strategy is inspired from the following
intuition: due to sample selection bias caused by the logging policy,
labels for some regions of the sample space may be less likely to
be observed in the logged data, thus increasing the uncertainty in
these regions. To counter this effect, during active learning, the
learner should query more labels from such regions.

We formalize this intuition as follows. Suppose we would like to design
a single query strategy $Q_{1}:\calX\rightarrow[0,1]$ that determines
the probability of querying the label for an instance during the active
learning phase. For any $Q_{1}$, we have the following generalization
error bound for learning with $n$ logged examples and $m$ unlabeled
examples from which the learner can select and query for labels (for
simplicity of illustration, we use the unclipped estimator here)
\begin{align*}
l(h_{1})-l(h_{2}) & \leq l(h_{1},S)-l(h_{2},S)+\frac{4\log\frac{2|\mathcal{H}|}{\delta}}{3(mq_{0}+n)}+\sqrt{4\E\frac{\One\{h_{1}(X)\neq h_{2}(X)\}}{mQ_{0}(X)+nQ_{1}(X)}\log\frac{2|\mathcal{H}|}{\delta}}.
\end{align*}

We propose to set $Q_{1}(x)=\One\{mQ_{0}(x)<\frac{m}{2}Q_{0}(x)+n\}$
which only queries instances if $Q_{0}(x)$ is small. This leads to
fewer queries while guarantees an error bound close to the one achieved
by setting $Q_{1}(x)\equiv1$ that queries every instance. In Appendix~\ref{sec:app-vcis}
we give an example, Example~\ref{exa:debias}, showing the reduction
of queries due to this strategy.

The \debias strategy is complementary to the DBAL technique. We note
that a similar query strategy is proposed in~\cite{YCJ18}, but the
strategy here stems from a tighter analysis and can be applied with
variance control techniques discussed in Section~\ref{sec:Bias-Variance-Trade-Off},
and thus gives better label complexity guarantees as to be discussed
in the analysis section.

%Besides,  has to use disjoint data sets to train the model in each iteration as their query strategy depends on the data, while here we can use all available data in each iteration as the proposed  has no dependency over the collected data.

\subsection{Algorithm}

Putting things together, our proposed algorithm is shown as Algorithm~\ref{alg:main}.
It takes the logged data and an epoch schedule as input. It assumes
the logging policy $Q_{0}$ and its distribution $f(x)=\Pr(Q_{0}(X)\leq x)$
are known (otherwise, these quantities can be estimated with unlabeled
data).

Algorithm~\ref{alg:main} uses the DBAL framework that recursively
shrinks a candidate set $\vs$ and its corresponding disagreement
region $D$ to save label queries by not querying examples outside
$D$. In particular, at iteration $k$, it computes a clipping threshold
$M_{k}$ (step~\ref{alg:Mk}) and MIS weights $w_{k}(x):=\frac{m+n_{k}}{mQ_{0}(X_{i})+\sum_{j=1}^{k}\tau_{i}Q_{i}(X_{i})}$
which are used to define the clipped MIS error estimator and two second
moment estimators
\begin{align*}
l(h;\tilde{S}_{k},M_{k}) & :=\frac{1}{m+n_{k}}\sum_{i=1}^{m+n_{k}}w_{k}(X_{i})Z_{i}\One\{h(X_{i})\neq\tilde{Y}_{i}\}\One\{w_{k}(X_{i})\leq M_{k}\},\\
\hvar(h_{1},h_{2};\tilde{S}_{k},M_{k}) & :=\frac{1}{m+n_{k}}\sum_{i=1}^{m+n_{k}}w_{k}^{2}(X_{i})Z_{i}\One\{h_{1}(X_{i})\neq h_{2}(X_{i})\}\One\{w_{k}(X_{i})\leq M_{k}\},\\
\hvar(h;\tilde{S}_{k},M_{k}) & :=\frac{1}{m+n_{k}}\sum_{i=1}^{m+n_{k}}w_{k}^{2}(X_{i})Z_{i}\One\{h(X_{i})\neq\tilde{Y}_{i}\}\One\{w_{k}(X_{i})\leq M_{k}\}.
\end{align*}
The algorithm shrinks the candidate set $\vs_{k+1}$ by eliminating
classifiers whose estimated error is larger than a threshold that
takes the minimum empirical error and the second moment into account
(step~\ref{alg:cand}), and defines a corresponding disagreement
region $D_{k+1}=\text{DIS}(\vs_{k+1})$ as the set of all instances
on which there are two classifiers in the candidate set $\vs_{k+1}$
that predict labels differently. It derives a query policy $Q_{k+1}$
with the \debias strategy (step~\ref{alg:Qk}). At the end of iteration
$k$, it draws $\tau_{k+1}$ unlabeled examples. For each example
$X$ with $Q_{k+1}(X)>0$, if $X\in D_{k+1}$, the algorithm queries
for the actual label $Y$ and sets $\tilde{Y}=Y$, otherwise it infers
the label and sets $\tilde{Y}=\hat{h}_{k}(X)$. These examples $\{X\}$
and their inferred or queried labels $\{\tilde{Y}\}$ are then used
in subsequent iterations. In the last step of the algorithm, a classifier
that minimizes the clipped MIS error with the second moment regularizer
over all received data is returned.

\begin{algorithm}[h]
\begin{algorithmic}[1]
\State{Input: confidence $\delta$, logged data $T_0$, epoch schedule $\tau_1,\dots,\tau_K$, $n=\sum_{i=1}^K\tau_i$.}
\State{$\tilde{S}_0 \gets T_0$; $\vs_0 \gets \mathcal{H}$; $D_0 \gets \mathcal{X}$; $n_0=0$}
\For{$k=0, \dots, K-1$}	
	\State{$\sigma_1(k, \delta, M) \gets (\frac{M}{m+n_k}+\frac{M^2}{(m+n_k)^\frac32})\log\frac{|\mathcal{H}|}{\delta}; \sigma_2(k,\delta)=\frac{1}{m+n_k}\log\frac{|\mathcal{H}|}{\delta}; \delta_k \gets \frac{\delta}{2(k+1)(k+2)}$}	
	\State{Choose $M_k = \inf\{M\geq 1\mid\frac{2M}{m+n_k}\log\frac{|\calH|}{\delta_k} \geq \Pr(\frac{m+n_k}{mQ_0(X)+n_k}> M/2)\}$}\label{alg:Mk}
	\State{$\hat{h}_k \gets \arg\min_{h\in \vs_k} l(h; \tilde{S}_k, M_k)$}
	\State{Define the candidate set $\vs_{k+1} \gets \{ h\in \vs_k \mid l(h;\tilde{S}_{k},M_k)\leq l(\hat{h}_{k};\tilde{S}_{k},M_k)+\gamma_1\sigma_{1}(k,\delta_{k}, M_k)+\gamma_1\sqrt{\sigma_{2}(k,\delta_{k})\hvar(h,\hat{h}_{k};\tilde{S}_{k},M_k)} \}$}\label{alg:cand}
	\State{Define the Disagreement Region $D_{k+1} \gets \{x\in\mathcal{X} \mid \exists h_1, h_2 \in \vs_{k+1}\text{ s.t. }h_1(x)\neq h_2(x)\}$}		
    \State{$Q_{k+1}(x)\gets \One\{mQ_0(x)+\sum_{i=1}^k\tau_iQ_i(x)<\frac{m}{2}Q_0(x)+n_{k+1}\}$;}\label{alg:Qk}
	\State{$n_{k+1} \gets n_k + \tau_{k+1}$}
    \State{Draw $\tau_{k+1}$ samples $\{(X_t,Y_t)\}_{t=m+n_k+1}^{m+n_{k+1}}$, and present $\{X_t\}_{t=m+n_k+1}^{m+n_{k+1}}$ to the learner.} 
	\For{$t=m+n_k+1 \text{ to } m+n_{k+1}$} 
		\State{$Z_t\gets Q_{k+1}(X_t)$}	 	
		\If{$Z_t=1$} 	
		\State{If $X_t \in D_{k+1}$, query for label: $\tilde{Y}_t\gets Y_t$; otherwise infer $\tilde{Y}_t \gets \hat{h}_k(X_t)$.} 	
		\EndIf 	
	\EndFor	 
	\State{$\tilde{T}_{k+1} \gets \{X_t, \tilde{Y}_t, Z_t\}_{t=m+n_k+1}^{m+n_{k+1}}$, $\tilde{S}_{k+1} \gets \tilde{S}_{k}\cup \tilde{T}_{k+1}$;} 
\EndFor
\State{Output $\hat{h}=\arg\min_{h\in \vs_{K}} l(h; \tilde{S}_{K}, M_k)+\gamma_1\sqrt{\frac{1}{m+n}\log\frac{|\mathcal{H}|}{\delta_{K}}\hvar(h;\tilde{S}_{K},M_k)}$.}\label{alg:final}
\end{algorithmic}

\caption{\label{alg:main}Disagreement-Based Active Learning with Logged Observational
Data}
\end{algorithm}

\subsection{\label{sec:Analysis}Analysis}

We have the following generalization error bound for Algorithm~\ref{alg:main}.
Despite not querying for all labels, our algorithm achieves the same
asymptotic bound as the one that queries labels for all online data.
\begin{thm}
\label{thm:Convergence} Let $M=\inf\{M'\geq1\mid\frac{2M'}{m+n}\log\frac{|\calH|}{\delta_{K}}\geq\Pr(\frac{m+n}{mQ_{0}(X)+n}\geq M'/2)\}$
be the final clipping threshold used in step~\ref{alg:final}. There
is an absolute constant $c_{0}>1$ such that for any $\delta>0$,
with probability at least $1-\delta$,
\[
l(\hat{h})\leq l(h^{\star})+c_{0}(\sqrt{\E\frac{\One\{h^{\star}(X)\neq Y\}}{mQ_{0}(X)+n}\One\{\frac{m+n}{mQ_{0}(X)+n}\leq M\}\log\frac{|\mathcal{H}|}{\delta}}+\frac{M\log\frac{|\mathcal{H}|}{\delta}}{m+n}+\frac{M^{2}\sqrt{\log\frac{|\mathcal{H}|}{\delta}}}{(m+n)^{\frac{3}{2}}}).
\]
\end{thm}

%(2) when , this bound reduces to classical error bound for DBAL. (3) As discussed in Section, we always have , and thus our principled choice of the clipping threshold always leads to a performance guarantee no worse than that without clipping.

Next, we analyze the number of labels queried by Algorithm~\ref{alg:main}
with the help of following definitions.
\begin{defn}
For any $t\geq1,r>0$, define the modified disagreement coefficient
$\tilde{\theta}(r,t):=\frac{1}{r}\Pr\left(\text{DIS}(B(h^{\star},r))\cap\left\{ x:Q_{0}(x)\leq\frac{1}{t}\right\} \right)$.
Define $\tilde{\theta}:=\sup_{r>2\nu}\tilde{\theta}(r,\frac{2m}{n})$.
\end{defn}

The modified disagreement coefficient $\tilde{\theta}(r,t)$ measures
the probability of the intersection of two sets: the disagreement
region for the $r$-ball around $h^{\star}$ and where the propensity
score $Q_{0}(x)$ is smaller than $\frac{1}{t}$. It characterizes
the size of the querying region of Algorithm~\ref{alg:main}. Note
that the standard disagreement coefficient \cite{H07}, which is widely
used for analyzing DBAL in the classical active learning setting,
can be written as $\theta(r):=\tilde{\theta}(r,1)$. Here, the modified
disagreement coefficient modifies the standard definition to account
for the reduction of the number of label queries due to the \debias
strategy: Algorithm~\ref{alg:main} only queries examples on which
$Q_{0}(x)$ is lower than some threshold, hence $\tilde{\theta}(r,t)\leq\theta(r)$.
Moreover, our modified disagreement coefficient $\tilde{\theta}$
is always smaller than the modified disagreement coefficient of~\cite{YCJ18}
(denoted by $\theta'$) which is used to analyze their algorithm.

Additionally, define $\alpha=\frac{m}{n}$ to be the size ratio of
logged and online data, let $\tau_{k}=2^{k}$, define $\xi=\min_{1\leq k\leq K}\{M_{k}/\frac{m+n_{k}}{mq_{0}+n_{k}}\}$
to be the minimum ratio between the clipping threshold $M_{k}$ and
maximum MIS weight $\frac{m+n_{k}}{mq_{0}+n_{k}}$ ($\xi\leq1$ since
$M_{k}\leq\frac{m+n_{k}}{mq_{0}+n_{k}}$ by the choice of $M_{k}$),
and define $\bar{M}=\max_{1\leq k\leq K}M_{k}$ to be the maximum
clipping threshold. Recall $q_{0}=\inf_{X}Q_{0}(X)$.

The following theorem upper-bounds the number of label queries by
Algorithm~\ref{alg:main}.
\begin{thm}
\label{thm:Label-Complexity}There is an absolute constant $c_{1}>1$
such that for any $\delta>0$, with probability at least $1-\delta$,
the number of labels queried by Algorithm~\ref{alg:main} is at most:

\[
c_{1}\tilde{\theta}\cdot(n\nu+\sqrt{\frac{n\nu\xi}{\alpha q_{0}+1}\log\frac{|\mathcal{H}|\log n}{\delta}}+\frac{\bar{M}\xi\log n}{\sqrt{n\alpha}}\sqrt{\log\frac{|\mathcal{H}|\log n}{\delta}}+\frac{\xi\log n}{\alpha q_{0}+1}\log\frac{|\mathcal{H}|\log n}{\delta}).
\]
\end{thm}

\subsection{Discussion}

In this subsection, we compare the theoretical performance of the proposed algorithm
and some alternatives to understand the effect of proposed techniques. We
present some empirical results in Section~\ref{sec:exp} in Appendix.

The theoretical performance of learning algorithms is captured by
label complexity, which is defined as the number of label queries
required during the active learning phase to guarantee the test error
of the output classifier to be at most $\nu+\epsilon$ (here $\nu=l(h^{\star})$
is the optimal error , and $\epsilon$ is the target excess error).
This can be derived by combining the upper bounds on the error (Theorem~\ref{thm:Convergence})
and the number of queries (Theorem~\ref{thm:Label-Complexity}).
\begin{itemize}
\item The label complexity is $\tilde{O}\left(\nu\tilde{\theta}\log|\calH|\cdot\left(\frac{M}{\epsilon(1+\alpha)}+\frac{1}{\epsilon^{2}}\E\frac{\One\{h^{\star}(X)\neq Y\}}{1+\alpha Q_{0}(X)}\One\{\frac{1+\alpha}{1+\alpha Q_{0}(X)}\leq M\}\right)\right)$
for Algorithm~\ref{alg:main}. This is derived from Theorem~\ref{thm:Convergence},
\ref{thm:Label-Complexity}.
\item The label complexity is $\tilde{O}\left(\nu\tilde{\theta}\log|\calH|\cdot\left(\frac{1}{\epsilon(1+\alpha q_{0})}+\frac{1}{\epsilon^{2}}\E\frac{\One\{h^{\star}(X)\neq Y\}}{1+\alpha Q_{0}(X)}\right)\right)$
without clipping. This is derived by setting the final clipping threshold
$M_{K}=\frac{1+\alpha}{1+\alpha q_{0}}$. It is worse since $\frac{1+\alpha}{1+\alpha q_{0}}\geq M$.
\item The label complexity is $\tilde{O}\left(\nu\tilde{\theta}\log|\calH|\cdot(\frac{1}{\epsilon}+\frac{\nu}{\epsilon^{2}})\frac{1}{1+\alpha q_{0}}\right)$
if regularizers are removed further. This is worse since $\frac{\nu}{1+\alpha q_{0}}\geq\E\frac{\One\{h^{\star}(X)\neq Y\}}{1+\alpha Q_{0}(X)}$.
\item The label complexity is $\tilde{O}\left(\nu\theta\log|\calH|\cdot(\frac{1}{\epsilon}+\frac{\nu}{\epsilon^{2}})\frac{1}{1+\alpha q_{0}}\right)$
if we further remove the \debias strategy. Here the standard disagreement
coefficient $\theta$ is used ($\theta\geq\tilde{\theta}$).
\item The label complexity is $\tilde{O}\left(\nu\theta\log|\calH|\cdot\left(\frac{1}{\epsilon(1+\alpha q_{0})}+\frac{\nu(q_{0}+\alpha)}{\epsilon^{2}(1+\alpha)^{2}q_{0}}\right)\right)$
if we further remove the MIS technique. It can be shown $\frac{q_{0}+\alpha}{(1+\alpha)^{2}q_{0}}\geq\frac{1}{1+\alpha q_{0}}$,
so MIS gives a better label complexity bound.
\item The label complexity is $\tilde{O}\left(\log|\calH|\cdot\left(\frac{1}{\epsilon(1+\alpha q_{0})}+\frac{\nu(q_{0}+\alpha)}{\epsilon^{2}(1+\alpha)^{2}q_{0}}\right)\right)$
if DBAL is further removed. Here, all $n$ online examples are queried.
This demonstrates that DBAL decreases the label complexity bound by
a factor of $\nu\theta$ which is at most 1 by definition.
\item Finally, the label complexity is $\tilde{O}\left(\nu\theta'\log|\calH|\cdot\frac{\nu+\epsilon}{\epsilon^{2}}\frac{1}{1+\alpha q_{0}}\right)$
for \cite{YCJ18}, the only known algorithm in our setting. Here,
$\theta'\geq\tilde{\theta}$, $\frac{\nu}{1+\alpha q_{0}}\geq\E\frac{\One\{h^{\star}(X)\neq Y\}}{1+\alpha Q_{0}(X)}$,
and $\frac{1}{1+\alpha q_{0}}\geq\frac{M}{1+\alpha}$. Thus, the label
complexity of the proposed algorithm is better than \cite{YCJ18}.
This improvement is made possible by the second moment regularizer,
the principled clipping technique, and thereby the improved \debias
strategy.
\end{itemize}

\section{Conclusion}

We consider active learning with logged observational data where the
learner is given an observational data set selected according to some
logging policy, and can actively query for additional labels from
an online data stream. Previous work applies disagreement-based active
learning with an importance weighted loss estimator to account for
counterfactuals, which has high variance and leads to a high label
complexity. In this work, we utilize variance control techniques for
importance weighted estimators, and propose a novel variant of DBAL
to make it amenable to \vcis. Based on these improvements, a new
\debias strategy is proposed to further boost label efficiency. Our
theoretical analysis shows that the proposed algorithm is statistically
consistent and more label-efficient than prior work and alternative
methods.

\paragraph*{Acknowledgement}
We thank NSF under CCF 1513883 and 1719133 for support.

\bibliographystyle{plain}
\bibliography{counterfactual,lpactive,al-bandit-2019}

\clearpage
\appendix
\section{Preliminaries}

\subsection{Summary of Key Notations}

\paragraph{Data}

$T_{0}=\{(X_{t},Y_{t},Z_{t})\}_{t=1}^{m}$ is the logged data. $\tilde{T}_{k}=\{(X_{t},\tilde{Y}_{t},Z_{t})\}_{t=m+n_{k-1}+1}^{m+n_{k}}$
($1\leq k\leq K$) is the online data collected in the $k$-th iteration
of size $\tau_{k}=n_{k}-n_{k-1}$, and $\tilde{Y}_{t}$ equals either
the actual label $Y_{t}$ drawn from the data distribution $D$ or
the inferred label $\hat{h}_{k-1}(X_{t})$ according to the candidate
set $\vs_{k-1}$at iteration $k-1$. $\tilde{S}_{k}=T_{0}\cup\tilde{T}_{1}\cup\cdots\cup\tilde{T}_{k}$.

For convenience, we additionally define $T_{k}=\{(X_{t},Y_{t},Z_{t})\}_{t=m+n_{k-1}+1}^{m+n_{k}}$
to be the data set with the actual labels $Y_{t}$ drawn from the
data distribution, and $S_{k}=T_{0}\cup T_{1}\cup\cdots\cup T_{k}$.
The algorithm only observes $\tilde{S}_{k}$ and $\tilde{T}_{k}$,
and $S_{k},T_{k}$ are used for analysis only.

For $1\leq k\leq K$,$n_{k}=\tau_{1}+\cdots+\tau_{k}$, and we define
$n_{0}=0$, $n=n_{K}$, $\tau_{0}=m$. We assume $\tau_{k}\leq\tau_{k+1}$
for $1\leq k<K$.

Recall that $\{(X_{t},Y_{t},Z_{t})\}_{t=1}^{m+n}$ is an independent
sequence, and furthermore $\{(X_{t},Y_{t})\}_{t=1}^{m+n}$ is an i.i.d.
sequence drawn from $D$. For $(X,Z)\in T_{k}$ ($0\leq k\leq K)$,
$Q_{k}(X)=\Pr(Z=1\mid X)$. Unless otherwise specified, all probabilities
and expectations are over the random draw of all random variables
$\{(X_{t},Y_{t},Z_{t})\}_{t=1}^{m+n}$.

\paragraph{Loss and Second Moment}

The test error $l(h)=\Pr(h(X)\neq Y)$, the optimal classifier $h^{\star}=\arg\min_{h\in\calH}l(h)$,
and the optimal error $\nu=l(h^{\star})$. At the $k$-th iteration,
the Multiple Importance Sampling (MIS) weight $w_{k}(x)=\frac{m+n_{k}}{mQ_{0}(X_{t})+\sum_{i=1}^{k}\tau_{i}Q_{i}(X_{t})}$.
The clipped MIS loss estimator $l(h;S_{k},M)=\frac{1}{m+n_{k}}\sum_{i=1}^{m+n_{k}}w_{k}(X_{i})Z_{i}\One\{h(X_{i})\neq Y_{i}\}\One\{w_{k}(X_{i})\leq M\}$.
The (unclipped) MIS loss estimator $l(h;S_{k})=l(h;S_{k},\infty)$.

The clipped second moment $\var(h;k,M)=\E\left[w_{k}(X)\One\{h(X)\neq Y\}\One\{w_{k}(X)\leq M\}\right]$,
$\var(h_{1},h_{2};k,M)=\E\left[w_{k}(X)\One\{h_{1}(X)\neq h_{2}(X)\}\One\{w_{k}(X)\leq M\}\right]$.
The clipped second-moment estimators $\hvar(h;S_{k},M)=\frac{1}{m+n_{k}}\sum_{i=1}^{m+n_{k}}w_{k}^{2}(X_{i})Z_{i}\One\{h(X_{i})\neq Y_{i}\}\One\{w_{k}(X_{i})\leq M\}$,
$\hvar(h_{1},h_{2};S_{k},M)=\frac{1}{m+n_{k}}\sum_{i=1}^{m+n_{k}}w_{k}^{2}(X_{i})Z_{i}\One\{h_{1}(X)\neq h_{2}(X)\}\One\{w_{k}(X_{i})\leq M\}$.
The unclipped second moments ($\var(h;k)$,$\var(h_{1},h_{2};k)$)
and second moment estimators ($\hvar(h;S_{k})$,$\hvar(h_{1},h_{2};S_{k})$)
are defined similarly.

\paragraph{Disagreement Regions}

The $r$-ball around $h$ is defined as $B(h,r):=\{h'\in\mathcal{H}\mid\Pr(h(X)\neq h'(X))\leq r\}$,
and the disagreement region of $\vs\subseteq\calH$ is $\text{DIS}(\vs):=\{x\in\mathcal{X}\mid\exists h_{1}\neq h_{2}\in\vs\text{ s.t. }h_{1}(x)\neq h_{2}(x)\}$.

The candidate set $\vs_{k}$ and its disagreement region $D_{k}$
are defined in Algorithm~\ref{alg:main}. The empirical risk minimizer
(ERM) at $k$-th iteration $\hat{h}_{k}=\arg\min_{h\in\vs_{k}}l(h,\tilde{S}_{k})$.

The modified disagreement coefficient $\tilde{\theta}(r,\alpha):=\frac{1}{r}\Pr\left(\text{DIS}(B(h^{\star},r))\cap\left\{ x:Q_{0}(x)\leq\frac{1}{\alpha}\right\} \right)$.
$\tilde{\theta}=\sup_{r>2\nu}\tilde{\theta}(r,\frac{2m}{n})$.

\paragraph{Other Notations}

$q_{0}=\inf_{x}Q_{0}(x)$. $Q_{k+1}(x)=\One\{mQ_{0}(x)+\sum_{i=1}^{k}\tau_{i}Q_{i}(x)<\frac{m}{2}Q_{0}(x)+n_{k+1}\}$.
$M_{k}=\inf\{M\geq1\mid\frac{2M}{m+n_{k}}\log\frac{|\calH|}{\delta_{k}}\geq\Pr(\frac{m+n_{k}}{mQ_{0}(X)+n_{k}}>M/2)\}$.
$\xi=\min_{1\leq k\leq K}\{M_{k}/\frac{m+n_{k}}{mq_{0}+n_{k}}\}$.
$\bar{M}=\max_{1\leq k\leq K}M_{k}$.

\subsection{Elementary Facts}
\begin{prop}
\label{prop:quad-ineq}Suppose $a,c\geq0$, $b\in\R$. If $a\leq b+\sqrt{ca}$,
then $a\leq2b+c$.
\end{prop}

\begin{proof}
Since $a\leq b+\sqrt{ca}$, $\sqrt{a}\leq\frac{\sqrt{c}+\sqrt{c+4b}}{2}\leq\sqrt{\frac{c+c+4b}{2}}=\sqrt{c+2b}$
where the second inequality follows from the Root-Mean Square-Arithmetic
Mean inequality. Thus, $a\leq2b+c$.
\end{proof}

\subsection{Facts on Disagreement Regions and Candidate Sets}
\begin{lem}
\label{lem:l-diff-S-S_tilde}For any $k=0,\dots,K$, $M\geq0$, if
$h_{1},h_{2}\in\vs_{k}$, then $l(h_{1};S_{k},M)-l(h_{2};S_{k},M)=l(h_{1};\tilde{S}_{k},M)-l(h_{2};\tilde{S}_{k},M)$
and $\hvar(h_{1},h_{2};S_{k},M)=\hvar(h_{1},h_{2};\tilde{S}_{k},M)$.
\end{lem}

\begin{proof}
For any $(X_{t},Y_{t},Z_{t})\in S_{k}$ that $Z_{t}=1$, if $X_{t}\in\text{DIS}(\vs_{k})$,
then $Y_{t}=\tilde{Y}_{t}$, so $\One\{h_{1}(X_{t})\neq Y_{t}\}-\One\{h_{2}(X_{t})\neq Y_{t}\}=\One\{h_{1}(X_{t})\neq\tilde{Y}_{t}\}-\One\{h_{2}(X_{t})\neq\tilde{Y}_{t}\}$.
If $X_{t}\notin\text{DIS}(\vs_{k})$, then $h_{1}(X_{t})=h_{2}(X_{t})$,
so $\One\{h_{1}(X_{t})\neq Y_{t}\}-\One\{h_{2}(X_{t})\neq Y_{t}\}=\One\{h_{1}(X_{t})\neq\tilde{Y}_{t}\}-\One\{h_{2}(X_{t})\neq\tilde{Y}_{t}\}=0$.
Thus, $l(h_{1};S_{k},M)-l(h_{2};S_{k},M)=l(h_{1};\tilde{S}_{k},M)-l(h_{2};\tilde{S}_{k},M).$

$\hvar(h_{1},h_{2};S_{k},M)=\hvar(h_{1},h_{2};\tilde{S}_{k},M)$ holds
since $\hvar(h_{1},h_{2};S_{k},M)$ and $\hvar(h_{1},h_{2};\tilde{S}_{k},M)$
do not involve labels $Y$ or $\tilde{Y}$.
\end{proof}
The following lemmas are immediate from the definition.
\begin{lem}
\label{lem:l-fav-bias}For any $1\leq k\leq K$, if $h\in\vs_{k}$,
then $l(h;\tilde{S}_{k},M)\leq l(h;S_{k},M)\leq l(h;S_{k})$, and
$\hvar(h;\tilde{S}_{k},M)\leq\hvar(h;S_{k},M)\leq\hvar(h;S_{k})$.
\end{lem}

\begin{rem}
The inequality on the second moment regularizer $\hvar$, which will
be used to prove the error bound (Theorem~\ref{thm:Convergence})
of Algorithm~\ref{alg:main}, is due to the decomposition property
$\hvar(h;S_{k},M)=\frac{|S_{k}\cap\text{DIS}(\vs_{k})|}{m+n_{k}}\hvar(h;S_{k}\cap\text{DIS}(\vs_{k}),M)+\frac{|S_{k}\cap\text{DIS}(\vs_{k})^{c}|}{m+n_{k}}\hvar(h;S_{k}\cap\text{DIS}(\vs_{k})^{c},M)$.
It does not hold for estimated variance $\hat{\text{Var}}(h;S_{k},M):=\hvar(h;S_{k},M)-l(h;S_{k},M)^{2}$.
This explains the necessity of introducing the second moment regularizer.
\end{rem}

\begin{lem}
\label{lem:dis-coefficient}For any $r\geq2\nu$, any $\alpha\geq1$,
$\Pr(\text{DIS}(B(h^{\star},r)\cap\{x:Q_{0}(x)\leq\frac{1}{\alpha}\})\leq r\tilde{\theta}(r,\alpha)$.
\end{lem}

\subsection{Facts on Multiple Importance Sampling Estimators}
\begin{prop}
\label{prop:mis-unbiased-general} Let $f:\calX\times\calY\rightarrow\R$.
For any $k$, the following equations hold:
\begin{align*}
\E[\frac{1}{m+n_{k}}\sum_{(X,Y,Z)\in S_{k}}w_{k}(X)Zf(X,Y)] & =\E[f(X,Y)],\\
\E[\frac{1}{m+n_{k}}\sum_{(X,Y,Z)\in S_{k}}w_{k}^{2}(X)Zf(X,Y)] & =\E[w_{k}(X)f(X,Y)].
\end{align*}
\end{prop}

\begin{proof}
\begin{align*}
\E[\sum_{(X,Y,Z)\in S_{k}}w_{k}(X)Zf(X,Y)] & =\sum_{i=0}^{k}\E[\sum_{(X,Y,Z)\in T_{i}}\E[w_{k}(X)f(X,Y)Z\mid X,Y]]\\
 & =\sum_{i=0}^{k}\E[\sum_{(X,Y,Z)\in T_{i}}w_{k}(X)f(X,Y)\E[Z\mid X,Y]]\\
 & \overset{(a)}{=}\sum_{i=0}^{k}\E[\sum_{(X,Y,Z)\in T_{i}}w_{k}(X)f(X,Y)\E[Z\mid X]]\\
 & =\sum_{i=0}^{k}\E[\sum_{(X,Y,Z)\in T_{i}}w_{k}(X)f(X,Y)Q_{i}(X)]\\
 & \overset{(b)}{=}\sum_{i=0}^{k}\tau_{i}\E[w_{k}(X)f(X,Y)Q_{i}(X)]\\
 & =\E[w_{k}(X)f(X,Y)\sum_{i=0}^{k}\tau_{i}Q_{i}(X)]\\
 & \overset{(c)}{=}(m+n_{k})\E[f(X,Y)]
\end{align*}
 where (a) follows from $\E[Z\mid X]=\E[Z\mid X,Y]$ as $Z,Y$ are
conditionally independent given $X$, (b) follows since $T_{i}$ is
a sequence of i.i.d. random variables, and (c) follows from the definition
$w_{k}(X)=\frac{m+n_{k}}{\sum_{i=0}^{k}\tau_{i}Q_{i}(X)}$.

The proof for the second equality is similar and skipped.
\end{proof}

\subsection{\label{subsec:Facts-Q}Facts on the \DEBIAS Query Strategy}

The query strategy $Q_{k}$ can be simplified as follows.
\begin{prop}
\label{prop:Q-simplified} For any $1\leq k\leq K$, $x\in\calX$,
$Q_{k}(x)=\One\{2n_{k}-mQ_{0}(x)>0\}$.
\end{prop}

\begin{proof}
The $k=1$ case can be easily verified. Suppose it holds for $Q_{k}$,
and we next show it holds for $\text{\ensuremath{Q_{k+1}}}$. Recall
by definition $Q_{k+1}(x)=\One\{mQ_{0}(x)+\sum_{i=1}^{k}\tau_{i}Q_{i}(x)<\frac{m}{2}Q_{0}(x)+n_{k+1}\}$.

If $Q_{k}(x)=1$, then $mQ_{0}(x)+\sum_{i=1}^{k-1}\tau_{i}Q_{i}(x)<\frac{m}{2}Q_{0}(x)+n_{k}$,
so
\begin{align*}
mQ_{0}(x)+\sum_{i=1}^{k}\tau_{i}Q_{i}(x) & <\frac{m}{2}Q_{0}(x)+n_{k}+\tau_{k}\\
 & \leq\frac{m}{2}Q_{0}(x)+n_{k+1}
\end{align*}
where the last inequality follows by the assumption on the epoch schedule
$\tau_{k}\leq\tau_{k+1}=n_{k+1}-n_{k}$. This implies $Q_{k+1}(x)=1$.
In this case, $\One\{2n_{k+1}-mQ_{0}(x)>0\}=1$ as well, since $n_{k+1}\geq n_{k}$
implies $2n_{k+1}-mQ_{0}(x)\geq2n_{k}-mQ_{0}(x)>0$.

The above argument also implies if $Q_{k}(x)=0$, then $Q_{1}(x)=Q_{2}(x)=\cdots=Q_{k-1}(x)=0$.
Thus, if $Q_{k}(x)=0$, then $Q_{k+1}(x)=\One\{mQ_{0}(x)<\frac{m}{2}Q_{0}(x)+n_{k+1}\}=\One\{2n_{k+1}-mQ_{0}(x)>0\}$.
\end{proof}
The following proposition gives an upper bound of the multiple importance
sampling weight, which will be used to bound the second moment of
the loss estimators with the \debias strategy.
\begin{prop}
\label{prop:Q-bound} For any $1\leq k\leq K$, $w_{k}(x)=\frac{m+n_{k}}{mQ_{0}(x)+\sum_{i=1}^{k}\tau_{i}Q_{i}(x)}\leq\frac{m+n_{k}}{\frac{1}{2}mQ_{0}(x)+n_{k}}$.
\end{prop}

\begin{proof}
The $k=1$ case can be easily verified. Suppose it holds for $w_{k}$,
and we next show it holds for $\text{\ensuremath{w_{k+1}}}$.

Now, if $Q_{k+1}(x)=0$, then by Proposition~\ref{prop:Q-simplified},
$2n_{k+1}-mQ_{0}(x)\leq0$, so $mQ_{0}(x)+\sum_{i=1}^{k+1}\tau_{i}Q_{i}(x)\geq mQ_{0}(x)\geq\frac{1}{2}mQ_{0}(x)+n_{k+1}$.

If $Q_{k+1}(x)=1$, then by the induction hypothesis, $mQ_{0}(x)+\sum_{i=1}^{k+1}\tau_{i}Q_{i}(x)\geq\frac{1}{2}mQ_{0}(x)+n_{k}+\tau_{k+1}=\frac{1}{2}mQ_{0}(x)+n_{k+1}$.

Thus, in both cases, $mQ_{0}(x)+\sum_{i=1}^{k+1}\tau_{i}Q_{i}(x)\geq\frac{1}{2}mQ_{0}(x)+n_{k+1}$,
so $w_{k+1}(x)\leq\frac{m+n_{k+1}}{\frac{1}{2}mQ_{0}(x)+n_{k+1}}$.
\end{proof}

\subsection{Lower Bound Techniques}

We present a lower bound for binomial distribution tails, which will
be used to prove generalization error lower bounds.
\begin{lem}
\label{lem:bin-lb} Let $0<t<p<1/2$, $B\sim\Bin(n,p)$ be a binomial
random variable, and $\delta=\sqrt{4n\frac{(t-p)^{2}}{p}}$. Then,
$\Pr(B<nt)\geq\frac{1}{\sqrt{2\pi}}\frac{\delta}{\delta^{2}+1}\exp(-\frac{1}{2}\delta^{2})$.
\end{lem}

This Lemma is a consequence of following lemmas.

\begin{lem}
Suppose $0<p,q<1$, $\text{KL}(p,q)=p\log\frac{p}{q}+(1-p)\log\frac{1-p}{1-q}$.
Then $\text{KL}(p,q)\leq\frac{(p-q)^{2}}{q(1-q)}$.
\end{lem}

\begin{proof}
Since $\log x\leq x-1$, $p\log\frac{p}{q}+(1-p)\log\frac{1-p}{1-q}\leq p(\frac{p}{q}-1)+(1-p)(\frac{1-p}{1-q}-1)=\frac{(p-q)^{2}}{q(1-q)}$.
\end{proof}
\begin{lem}
(\cite{BS79}) Suppose $X\sim N(0,1)$, and define $\Phi(t)=\Pr(X\leq t)$.
If $t>0$, then $\Phi(-t)\geq\frac{1}{\sqrt{2\pi}}\frac{t}{t^{2}+1}\exp(-\frac{1}{2}t^{2})$.
\end{lem}

\begin{lem}
(\cite{ZS13}) Let $B\sim\Bin(n,p)$ be a binomial random variable
and $0<k<np$. Then, $\Pr(B<k)\geq\Phi(-\sqrt{2n\text{KL}(\frac{k}{n},p)})$.
\end{lem}

\section{Deviation Bounds}

In this section, we demonstrate deviation bounds for our error estimators
on $S_{k}$.

We use following Bernstein-style concentration bound:
\begin{fact}
\label{fact:bernstein}Suppose $X_{1},\dots,X_{n}$ are independent
random variables such that $|X_{i}|\leq M$. Then with probability
at least $1-\delta$, 
\[
\left|\frac{1}{n}\sum_{i=1}^{n}X_{i}-\frac{1}{n}\sum_{i=1}^{n}\E X_{i}\right|\leq\frac{2M}{3n}\log\frac{2}{\delta}+\sqrt{\frac{2}{n^{2}}\sum_{i=1}^{n}\E X_{i}^{2}\log\frac{2}{\delta}}.
\]
\end{fact}

\begin{thm}
\label{thm:gen}For any $k=0,\dots,K$, any $\delta>0$, if $\frac{2M\log\frac{|\calH|}{\delta}}{m+n_{k}}\geq\Pr(\frac{m+n_{k}}{mQ_{0}(X)+n_{k}}\geq\frac{M}{2})$,
then with probability at least $1-\delta$, for all $h_{1},h_{2}\in\mathcal{H},$
the following statements hold simultaneously: 
\begin{align}
\left|\left(l(h_{1};S_{k},M)-l(h_{2};S_{k},M)\right)-\left(l(h_{1})-l(h_{2})\right)\right| & \leq\frac{10\log\frac{2|\mathcal{H}|}{\delta}}{3(m+n_{k})}M+\sqrt{\frac{4\log\frac{2|\mathcal{H}|}{\delta}}{m+n_{k}}\var(h_{1},h_{2};k,M)};\label{eq:thm-gen-diff-var}\\
\left|l(h_{1};S_{k},M)-l(h_{1})\right| & \leq\frac{10\log\frac{2|\mathcal{H}|}{\delta}}{3(m+n_{k})}M+\sqrt{\frac{4\log\frac{2|\mathcal{H}|}{\delta}}{m+n_{k}}\var(h_{1};k,M)}.\label{eq:thm-gen-l-var}
\end{align}
\end{thm}

\begin{proof}
We show proof for $k>0$. The $k=0$ case can be proved similarly.

First, define the clipped expected loss $l(h;k,M)=\E[\One\{h(X)\neq Y\}\One\{w_{k}(X)\leq M\}]$.
We have
\begin{align}
 & \left|\left(l(h_{1})-l(h_{2})\right)-\left(l(h_{1};k,M)-l(h_{2};k,M)\right)\right|\nonumber \\
= & \left|\E\left[(\One\{h_{1}(X)\neq Y\}-\One\{h_{2}(X)\neq Y\})\One\{w_{k}(X)>M\}\right]\right|\nonumber \\
\leq & \E\left[\One[w_{k}(X)>M]\right]\nonumber \\
\leq & \E[\One\{\frac{m+n_{k}}{mQ_{0}(X)+n_{k}}>\frac{M}{2}\}]\nonumber \\
\leq & \frac{2M}{m+n_{k}}\log\frac{|\calH|}{\delta}\label{eq:thm-gen-tmp-clip-gap}
\end{align}
 where the second inequality follows from Proposition~\ref{prop:Q-bound},
and the last inequality follows from the assumption on $M$.

Next, we bound $\left(l(h_{1};S_{k},M)-l(h_{2};S_{k},M)\right)-\left(l(h_{1};k,M)-l(h_{2};k,M)\right)$.

For any fixed $h_{1},h_{2}\in\mathcal{H}$, define $N:=|S_{k}|$,
$U_{t}:=w_{k}(X_{t})Z_{t}\One\{w_{k}(X_{t})\leq M\}(\One\{h_{1}(X_{t})\neq Y_{t}\}-\One\{h_{2}(X_{t})\neq Y_{t}\})$.

Now, $\{U_{t}\}_{t=1}^{N}$ is an independent sequence. $\frac{1}{N}\sum_{t=1}^{N}U_{t}=l(h_{1};S_{k},M)-l(h_{2};S_{k},M)$,
and $\E\frac{1}{N}\sum_{t=1}^{N}U_{t}=l(h_{1};k,M)-l(h_{2};k,M)$
by Proposition~\ref{prop:mis-unbiased-general}. Moreover, since
$(\One\{h_{1}(X_{t})\neq Y_{t}\}-\One\{h_{2}(X_{t})\neq Y_{t}\})^{2}=\One\{h_{1}(X_{t})\neq h_{2}(X_{t})\}$,
we have $\frac{1}{N}\sum_{t=1}^{N}U_{t}^{2}=\hvar(h_{1},h_{2};S_{k},M)$
and $\E\frac{1}{N}\sum_{t=1}^{N}U_{t}^{2}=\var(h_{1},h_{2};k,M)$
by Proposition~\ref{prop:mis-unbiased-general}. Applying Bernstein's
inequality (Fact~\ref{fact:bernstein}) to $\{U_{t}\}$, we have
with probability at least $1-\frac{\delta}{2}$, 
\[
\left|\frac{1}{N}\sum_{t=1}^{N}U_{t}-\E\frac{1}{N}\sum_{t=1}^{N}U_{t}\right|\leq\frac{2M}{3N}\log\frac{4}{\delta}+\sqrt{\frac{2}{N}\var(h_{1},h_{2};k,M)\log\frac{4}{\delta}},
\]
so $\left|\left(l(h_{1};S_{k},M)-l(h_{2};S_{k},M)\right)-\left(l(h_{1};k,M)-l(h_{2};k,M)\right)\right|\leq\frac{2M}{3(m+n_{k})}\log\frac{4}{\delta}+\sqrt{\frac{2}{m+n_{k}}\var(h_{1},h_{2};k,M)\log\frac{4}{\delta}}$.
By a union bound over $\calH$, with probability at least $1-\frac{\delta}{2}$
for all $h_{1},h_{2}\in\calH$,

\begin{align}
 & \left|\left(l(h_{1};S_{k},M)-l(h_{2};S_{k},M)\right)-\left(l(h_{1};k,M)-l(h_{2};k,M)\right)\right|\nonumber \\
 & \leq\frac{4M}{3(m+n_{k})}\log\frac{2|\calH|}{\delta}+\sqrt{\frac{4}{m+n_{k}}\var(h_{1},h_{2};k,M)\log\frac{2|\calH|}{\delta}}.\label{eq:thm-gen-tmp-clipped-gen}
\end{align}

(\ref{eq:thm-gen-diff-var}) follows by combining (\ref{eq:thm-gen-tmp-clip-gap})
and (\ref{eq:thm-gen-tmp-clipped-gen}).

The proof for (\ref{eq:thm-gen-l-var}) is similar and skipped.
\end{proof}
We use following bound for the second moment which is an immediate
corollary of Lemmas~B.1~and~B.2 in \cite{ND17}:
\begin{fact}
\label{fact:var-deviation}Suppose $X_{1},\dots,X_{n}$ are independent
random variables such that $|X_{i}|\leq M$. Then with probability
at least $1-\delta$,

\[
-\sqrt{\frac{2M^{2}}{n}\log\frac{1}{\delta}}-\frac{M^{2}}{n}\leq\sqrt{\frac{1}{n}\sum_{i=1}^{n}X_{i}^{2}}-\sqrt{\E\frac{1}{n}\sum_{i=1}^{n}X_{i}^{2}}\leq\sqrt{\frac{2M^{2}}{n}\log\frac{1}{\delta}}.
\]
\end{fact}

Recall by Lemma~\ref{prop:mis-unbiased-general}, $\E[\hvar(h_{1},h_{2};S_{k},M)]=\var(h_{1},h_{2};k,M)$
and $\E[\hvar(h_{1};S_{k},M)]=\var(h_{1};k,M)$. The following Corollary
follows from the bound on the second moment.
\begin{cor}
\label{cor:var-gen}For any $k=0,\dots,K$, any $\delta,M>0$, with
probability at least $1-\delta$, for all $h_{1},h_{2}\in\mathcal{H},$
the following statements hold:
\end{cor}

\begin{equation}
\left|\sqrt{\hvar(h_{1},h_{2};S_{k},M)}-\sqrt{\var(h_{1},h_{2};k,M)}\right|\leq\sqrt{\frac{4M^{2}}{m+n_{k}}\log\frac{2|\mathcal{H}|}{\delta}}+\frac{M^{2}}{m+n_{k}},\label{eq:cor-var-gen-diff-var}
\end{equation}

\begin{equation}
\left|\sqrt{\hvar(h_{1};S_{k},M)}-\sqrt{\var(h_{1};k,M)}\right|\leq\sqrt{\frac{4M^{2}}{m+n_{k}}\log\frac{2|\mathcal{H}|}{\delta}}+\frac{M^{2}}{m+n_{k}}.\label{eq:cor-var-gen-var}
\end{equation}

\begin{cor}
\label{cor:gen}There is an absolute constant $\gamma_{1}$, for any
$k=0,\dots,K$, any $\delta>0$, if $\frac{2M\log\frac{|\calH|}{\delta}}{m+n_{k}}\geq\Pr(\frac{m+n_{k}}{mQ_{0}(X)+n_{k}}\geq\frac{M}{2})$,
then with probability at least $1-\delta$, for all $h_{1},h_{2}\in\mathcal{H},$
the following statements hold: 
\begin{align}
\left|\left(l(h_{1};S_{k},M)-l(h_{2};S_{k},M)\right)-\left(l(h_{1})-l(h_{2})\right)\right|\leq & \gamma_{1}\frac{M}{m+n_{k}}\log\frac{|\mathcal{H}|}{\delta}+\gamma_{1}\frac{M^{2}}{(m+n_{k})^{\frac{3}{2}}}\sqrt{\log\frac{|\mathcal{H}|}{\delta}}\label{eq:cor-gen-diff-var_S}\\
 & +\gamma_{1}\sqrt{\frac{\log\frac{|\mathcal{H}|}{\delta}}{m+n_{k}}\hvar(h_{1},h_{2};S_{k},M)};\nonumber 
\end{align}

\begin{equation}
l(h_{1};S_{k},M)\leq2l(h_{1})+\gamma_{1}\frac{M}{m+n_{k}}\log\frac{|\mathcal{H}|}{\delta}.\label{eq:cor-gen-l-mul}
\end{equation}
\end{cor}

\begin{proof}
Let event $E$ be the event that (\ref{eq:thm-gen-diff-var}), (\ref{eq:thm-gen-l-var}),
and (\ref{eq:cor-var-gen-diff-var}) hold for all $h_{1},h_{2}\in\mathcal{H}$
with confidence $1-\frac{\delta}{3}$ respectively. Assume $E$ happens
(whose probability is at least $1-\delta$).

(\ref{eq:cor-gen-diff-var_S}) is immediate from (\ref{eq:thm-gen-diff-var})
and (\ref{eq:cor-var-gen-diff-var}).

For the proof of (\ref{eq:cor-gen-l-mul}), apply (\ref{eq:thm-gen-l-var})
to $h_{1}$, we get 
\[
l(h_{1};S_{k},M)\leq l(h_{1})+\frac{10\log\frac{6|\mathcal{H}|}{\delta}}{3(m+n_{k})}M+\sqrt{\frac{4\log\frac{6|\mathcal{H}|}{\delta}}{m+n_{k}}\var(h_{1};k,M)}.
\]

Now, $\var(h_{1};k,M)=\E\left[w_{k}(X)\One\{h_{1}(X)\neq Y\}\One\{w_{k}(X)\leq M\}\right]\leq M\E[\One\{h_{1}(X)\neq Y\}]$,
so $\sqrt{\frac{4\log\frac{6|\mathcal{H}|}{\delta}}{m+n_{k}}\var(h_{1};k,M)}\leq\sqrt{\frac{4M\log\frac{6|\mathcal{H}|}{\delta}}{m+n_{k}}l(h_{1})}\leq l(h_{1})+\frac{M\log\frac{6|\mathcal{H}|}{\delta}}{(m+n_{k})}$
where the last inequality follows from $\sqrt{ab}\leq\frac{a+b}{2}$
for $a,b\geq0$, and (\ref{eq:cor-gen-l-mul}) thus follows.
\end{proof}

\section{Technical Lemmas for Disagreement-Based Active Learning}

For any $0\leq k<K$ and $\delta>0$, define event $\mathcal{E}_{k,\delta}$
to be the event that the conclusions of Theorem~\ref{thm:gen} and
Corollary~\ref{cor:var-gen} hold for $k$ with confidence $1-\delta/2$
respectively. We have $\Pr(\mathcal{E}_{k,\delta})\geq1-\delta$,
and that $\mathcal{E}_{k,\delta}$ implies inequalities (\ref{eq:cor-gen-diff-var_S})
and (\ref{eq:cor-gen-l-mul}).

Recall that $\sigma_{1}(k,\delta,M)=\frac{M}{m+n_{k}}\log\frac{|\mathcal{H}|}{\delta}+\frac{M^{2}}{(m+n_{k})^{\frac{3}{2}}}\sqrt{\log\frac{|\mathcal{H}|}{\delta}};\sigma_{2}(k,\delta)=\frac{1}{m+n_{k}}\log\frac{|\mathcal{H}|}{\delta};\delta_{k}=\frac{\delta}{2(k+1)(k+2)}$.

We first present a lemma which can be used to guarantee that $h^{\star}$
stays in candidate sets with high probability by induction.
\begin{lem}
\label{lem:h_star_in}For any $k=0,\dots K$, any $\delta>0$, any
$M\geq1$ such that $\frac{2M\log\frac{|\calH|}{\delta}}{m+n_{k}}\geq\Pr(\frac{m+n_{k}}{mQ_{0}(X)+n_{k}}\geq\frac{M}{2})$,
on event $\mathcal{E}_{k,\delta}$, if $h^{\star}\in\vs_{k}$, then,
\[
l(h^{\star};\tilde{S}_{k},M)\leq l(\hat{h}_{k};\tilde{S}_{k},M)+\gamma_{1}\sigma_{1}(k,\delta,M)+\gamma_{1}\sqrt{\sigma_{2}(k,\delta)\hvar(h^{\star},\hat{h}_{k};\tilde{S}_{k},M)}.
\]
\end{lem}

\begin{proof}
\begin{align*}
 & l(h^{\star};\tilde{S}_{k},M)-l(\hat{h}_{k};\tilde{S}_{k},M)\\
= & l(h^{\star};S_{k},M)-l(\hat{h}_{k};S_{k},M)\\
\leq & \gamma_{1}\sigma_{1}(k,\delta,M)+\gamma_{1}\sqrt{\sigma_{2}(k,\delta)\hvar(h^{\star},\hat{h}_{k};S_{k},M)}\\
= & \gamma_{1}\sigma_{1}(k,\delta,M)+\gamma_{1}\sqrt{\sigma_{2}(k,\delta)\hvar(h^{\star},\hat{h}_{k};\tilde{S}_{k},M)}
\end{align*}

The first and the second equalities follow by Lemma~\ref{lem:l-diff-S-S_tilde}.
The inequality follows by Corollary~\ref{cor:gen}.
\end{proof}
Next, we present a lemma to bound the probability mass of the disagreement
region of candidate sets.
\begin{lem}
\label{lem:dis-radius} Let $\hat{h}_{k,M}=\arg\min_{h\in\vs_{k}}l(h;\tilde{S}_{k},M)$,
and $\vs_{k+1}(\delta,M):=\{h\in\vs_{k}\mid l(h;\tilde{S}_{k},M)\leq l(\hat{h}_{k,M};\tilde{S}_{k},M)+\gamma_{1}\sigma_{1}(k,\delta,M)+\gamma_{1}\sqrt{\sigma_{2}(k,\delta)\hvar(h,\hat{h}_{k,M};\tilde{S}_{k},M)}\}$.
There is an absolute constant $\gamma_{2}>1$ such that for any $k=0,\dots,K$,
any $\delta>0$, any $M\geq1$ such that $\frac{2M\log\frac{|\calH|}{\delta}}{m+n_{k}}\geq\Pr(\frac{m+n_{k}}{mQ_{0}(X)+n_{k}}\geq\frac{M}{2})$,
on event $\mathcal{E}_{k,\delta}$, if $h^{\star}\in\vs_{k}$, then
for all $h\in\vs_{k+1}(\delta,M)$, 
\[
l(h)-l(h^{\star})\leq\gamma_{2}\sigma_{1}(k,\delta,M)+\gamma_{2}\sqrt{\sigma_{2}(k,\delta)Ml(h^{\star})}.
\]
\end{lem}

\begin{proof}
For any $h\in\vs_{k+1}(\delta,M)$, we have 
\begin{align}
 & l(h)-l(h^{\star})\nonumber \\
\leq & l(h;S_{k},M)-l(h^{\star};S_{k},M)+\frac{10M\log\frac{4|\mathcal{H}|}{\delta}}{3(m+n_{k})}+\sqrt{4\frac{\var(h^{\star},h;k,M)}{m+n_{k}}\log\frac{4|\mathcal{H}|}{\delta}}\nonumber \\
= & l(h;\tilde{S}_{k},M)-l(h^{\star};\tilde{S}_{k},M)+\frac{10M\log\frac{4|\mathcal{H}|}{\delta}}{3(m+n_{k})}+\sqrt{4\frac{\var(h^{\star},h;k,M)}{m+n_{k}}\log\frac{4|\mathcal{H}|}{\delta}}\nonumber \\
= & l(h;\tilde{S}_{k},M)-l(\hat{h}_{k,M};\tilde{S}_{k},M)+l(\hat{h}_{k,M};\tilde{S}_{k},M)-l(h^{\star};\tilde{S}_{k},M)+\frac{10M\log\frac{4\mathcal{H}|}{\delta}}{3(m+n_{k})}+\sqrt{4\frac{\var(h^{\star},h;k,M)}{m+n_{k}}\log\frac{4|\mathcal{H}|}{\delta}}\nonumber \\
\leq & \gamma_{1}\sigma_{1}(k,\delta,M)+\gamma_{1}\sqrt{\sigma_{2}(k,\delta)\hvar(h,\hat{h}_{k,M};\tilde{S}_{k},M)}+\frac{10M\log\frac{4|\mathcal{H}|}{\delta}}{3(m+n_{k})}+\sqrt{4\frac{\var(h^{\star},h;k,M)}{m+n_{k}}\log\frac{4|\mathcal{H}|}{\delta}}\label{eq:dis-radius-pf-t1}
\end{align}
where the first equality follows from Lemma~\ref{lem:l-diff-S-S_tilde},
the first inequality follows from Theorem~\ref{thm:gen}, and the
second inequality follows from the definition of $\vs_{k}(\delta,M)$
and that $l(\hat{h}_{k,M};\tilde{S}_{k},M)\leq l(h^{\star};\tilde{S}_{k},M)$.

Next, we upper bound $\sqrt{\hvar(h,\hat{h}_{k,M};\tilde{S}_{k},M)}$.
We have 
\begin{align*}
\sqrt{\hvar(h,\hat{h}_{k,M};\tilde{S}_{k},M)} & \leq\sqrt{\hvar(h,h^{\star};\tilde{S}_{k},M)+\hvar(h^{\star},\hat{h}_{k,M};\tilde{S}_{k},M)}\\
 & \leq\sqrt{\hvar(h,h^{\star};\tilde{S}_{k},M)}+\sqrt{\hvar(h^{\star},\hat{h}_{k,M};\tilde{S}_{k},M)}
\end{align*}
where the first inequality follows from the triangle inequality that
$\hvar(h,\hat{h}_{k,M};\tilde{S}_{k},M)\leq\hvar(h,h^{\star};\tilde{S}_{k},M)+\hvar(h^{\star},\hat{h}_{k,M};\tilde{S}_{k},M)$
and the second follows from the fact that $\sqrt{a+b}\leq\sqrt{a}+\sqrt{b}$
for $a,b\geq0$.

For the first term, we have $\sqrt{\hvar(h,h^{\star};\tilde{S}_{k},M)}=\sqrt{\hvar(h,h^{\star};S_{k},M)}\leq\sqrt{\var(h,h^{\star};k,M)}+\sqrt{\frac{4M^{2}}{m+n_{k}}\log\frac{4|\mathcal{H}|}{\delta}}+\frac{M^{2}}{m+n_{k}}$
by Corollary~\ref{cor:var-gen}.

For the second term, we have

\begin{align*}
\sqrt{\hvar(h^{\star},\hat{h}_{k,M};\tilde{S},M)} & \leq\sqrt{M(l(h^{\star};\tilde{S}_{k},M)+l(\hat{h}_{k,M};\tilde{S}_{k},M))}\\
 & \leq\sqrt{2Ml(h^{\star};\tilde{S}_{k},M)}\\
 & \leq\sqrt{2Ml(h^{\star};S_{k},M)}\\
 & \leq\sqrt{2M(2l(h^{\star})+\gamma_{1}\frac{M}{m+n_{k}}\log\frac{|{\cal H}|}{\delta})}\\
 & \leq\sqrt{\frac{2\gamma_{1}M^{2}}{m+n_{k}}\log\frac{|\calH|}{\delta}}+2\sqrt{Ml(h^{\star})}
\end{align*}
where the first inequality follows since $w_{k}^{2}(X)Z\One\{h^{\star}(X)\neq\hat{h}_{k,M}(X)\}\One[w_{k}(X)\leq M]\leq M(w_{k}(X)Z\One\{h^{\star}(X)\neq Y\}+w_{k}(X)Z\One\{\hat{h}_{k,M}(X)\neq Y\})$,
the second inequality follows since $l(\hat{h}_{k,M};\tilde{S}_{k},M)\leq l(h^{\star};\tilde{S}_{k},M)$,
the third follows by Lemma~\ref{lem:l-fav-bias} since we assume
$h^{\star}\in\vs_{k}$, the fourth follows by Corollary~\ref{cor:gen},
and the last follows by $\sqrt{a+b}\leq\sqrt{a}+\sqrt{b}$.

Therefore, $\sqrt{\hvar(h,\hat{h}_{k,M};\tilde{S}_{k},M)}\leq\sqrt{\var(h,h^{\star};k,M)}+(2+\sqrt{2\gamma_{1}})\sqrt{\frac{M^{2}}{m+n_{k}}\log\frac{4|\mathcal{H}|}{\delta}}+\frac{M^{2}}{m+n_{k}}+2\sqrt{Ml(h^{\star})}$.
Continuing (\ref{eq:dis-radius-pf-t1}), we have

\begin{align*}
l(h)-l(h^{\star}) & \leq(\frac{10}{3}+3\gamma_{1}+2\sqrt{2}\gamma_{1}^{\frac{3}{2}})\frac{M}{m+n_{k}}\log\frac{4|\mathcal{H}|}{\delta}+\gamma_{1}\frac{M^{2}}{(m+n_{k})^{\frac{3}{2}}}\sqrt{\log\frac{4|\mathcal{H}|}{\delta}}\\
 & +(\gamma_{1}+2)\sqrt{\frac{\var(h^{\star},h;k,M)}{m+n_{k}}\log\frac{4|\mathcal{H}|}{\delta}}+2\gamma_{1}\sqrt{\frac{Ml(h^{\star})}{m+n_{k}}\log\frac{4|\mathcal{H}|}{\delta}}.
\end{align*}

Now, since $w_{k}^{2}(X)Z\One\{h^{\star}(X)\neq\hat{h}_{k}(X)\}\One[w_{k}(X)\leq M]\leq M(w_{k}(X)Z\One\{h^{\star}(X)\neq Y\}+w_{k}(X)Z\One\{\hat{h}_{k}(X)\neq Y\})$,
we have $\sqrt{\frac{\var(h^{\star},h;k,M)}{m+n_{k}}\log\frac{4|\mathcal{H}|}{\delta}}\leq\sqrt{\frac{M(l(h)-l(h^{\star})+2l(h^{\star}))}{m+n_{k}}\log\frac{4|\mathcal{H}|}{\delta}}\leq\sqrt{\frac{M(l(h)-l(h^{\star}))}{m+n_{k}}\log\frac{4|\mathcal{H}|}{\delta}}+\sqrt{\frac{2Ml(h^{\star})}{m+n_{k}}\log\frac{4|\mathcal{H}|}{\delta}}$
where the second follows by $\sqrt{a+b}\leq\sqrt{a}+\sqrt{b}$ for
$a,b\geq0$.

Thus, $l(h)-l(h^{\star})\leq(\frac{10}{3}+3\gamma_{1}+2\sqrt{2}\gamma_{1}^{\frac{3}{2}})\frac{M}{m+n_{k}}\log\frac{4|\mathcal{H}|}{\delta}+\gamma_{1}\frac{M^{2}}{(m+n_{k})^{\frac{3}{2}}}\sqrt{\log\frac{4|\mathcal{H}|}{\delta}}+(2\gamma_{1}+\sqrt{2}\gamma_{1}+2\sqrt{2})\sqrt{\frac{Ml(h^{\star})}{m+n_{k}}\log\frac{4|\mathcal{H}|}{\delta}}+(\gamma_{1}+2)\sqrt{\frac{M(l(h)-l(h^{\star}))}{m+n_{k}}\log\frac{4|\mathcal{H}|}{\delta}}$.

The result follows by applying Lemma~\ref{prop:quad-ineq} to $l(h)-l(h^{\star})$.
\end{proof}

\section{Proofs for Section~\ref{sec:Analysis}}
\begin{proof}
(of Theorem~\ref{thm:Convergence}) Define event $\mathcal{E}^{(0)}:=\bigcap_{k=0}^{K}\mathcal{E}_{k,\delta_{k}}$.
By a union bound, $\Pr(\mathcal{E}^{(0)})\geq1-\delta/2$. On event
$\mathcal{E}^{(0)}$, by induction and Lemma~\ref{lem:h_star_in},
for all $k=0,\dots,K$, $h^{\star}\in\vs_{k}$.

\begin{align*}
l(\hat{h})-l(h^{\star})\leq & l(\hat{h};S_{K},M_{K})-l(h^{\star};S_{K},M_{K})+\gamma_{1}\sigma_{1}(K,\delta_{K},M_{K})+\gamma_{1}\sqrt{\sigma_{2}(K,\delta_{K})\hvar(\hat{h},h^{\star};S_{K},M_{K})}\\
= & l(\hat{h};\tilde{S}_{K},M_{K})-l(h^{\star};\tilde{S}_{K},M_{K})+\gamma_{1}\sigma_{1}(K,\delta_{K},M_{K})+\gamma_{1}\sqrt{\sigma_{2}(K,\delta_{K})\hvar(\hat{h},h^{\star};\tilde{S}_{K},M_{K})}\\
\leq & l(\hat{h};\tilde{S}_{K},M_{K})+\gamma_{1}\sqrt{\sigma_{2}(K,\delta_{K})\hvar(\hat{h};\tilde{S}_{K},M_{K})}-l(h^{\star};\tilde{S}_{K},M_{K})-\gamma_{1}\sqrt{\sigma_{2}(K,\delta_{K})\hvar(h^{\star};\tilde{S}_{K},M_{K})}\\
 & +\gamma_{1}\sigma_{1}(K,\delta_{K},M_{K})+2\gamma_{1}\sqrt{\sigma_{2}(K,\delta_{K})\hvar(h^{\star};\tilde{S}_{K},M_{K})}\\
\leq & \gamma_{1}\sigma_{1}(K,\delta_{K},M_{K})+2\gamma_{1}\sqrt{\sigma_{2}(K,\delta_{K})\hvar(h^{\star};\tilde{S}_{K},M_{K})}\\
\leq & \gamma_{1}\sigma_{1}(K,\delta_{K},M_{K})+2\gamma_{1}\sqrt{\sigma_{2}(K,\delta_{K})\hvar(h^{\star};S_{K},M_{K})}\\
\leq & 3\gamma_{1}\sigma_{1}(K,\delta_{K},M_{K})+2\gamma_{1}\sqrt{\sigma_{2}(K,\delta_{K})\var(h^{\star};K,M_{K})}
\end{align*}
where the equality follows from Lemma~\ref{lem:l-diff-S-S_tilde},
the first inequality follows from Corollary~\ref{cor:gen}, the second
follows as $\sqrt{\hvar(\hat{h},h^{\star};\tilde{S}_{K},M_{K})}\leq\ensuremath{\sqrt{\hvar(\hat{h};\tilde{S}_{K},M_{K})+\hvar(h^{\star};\tilde{S}_{K},M_{K})}}\leq\text{\ensuremath{\sqrt{\hvar(\hat{h};\tilde{S}_{K},M_{K})}}+\ensuremath{\sqrt{\hvar(h^{\star};\tilde{S}_{K},M_{K})}}}$,
the third follows from the definition of $\hat{h}$, the forth follows
from Lemma~\ref{lem:l-fav-bias}, and the last follows from Corollary~\ref{cor:var-gen}.
\end{proof}
\begin{proof}
(of Theorem~\ref{thm:Label-Complexity}) Define event $\mathcal{E}^{(0)}:=\bigcap_{k=0}^{K}\mathcal{E}_{k,\delta_{k}}$.
On this event, by induction and Lemma~\ref{lem:h_star_in}, for all
$k=0,\dots,K-1$, $h^{\star}\in\vs_{k}$, and consequently by Lemma~\ref{lem:dis-radius},
$D_{k+1}\subseteq\text{DIS}(B(h^{\star},2\nu+\epsilon_{k}))$ where
$\epsilon_{k}=\gamma_{2}\sigma_{1}(k,\delta_{k},M_{k})+\gamma_{2}\sqrt{\sigma_{2}(k,\delta_{k})M_{k}\nu}$.

For any $k=0,\dots K-1$, the number of label queries at iteration
$k$ is $U_{k}:=\sum_{t=m+n_{k}+1}^{m+n_{k+1}}Z_{t}\One\{X_{t}\in D_{k+1}\}$
where the RHS is a sum of i.i.d. Bernoulli random variables with expectation
$\E[Z_{t}\One\{X_{t}\in D_{k+1}\}]=\Pr(D_{k+1}\cap\{x:Q_{0}(x)<\frac{2n_{k+1}}{m}\})$
since $Z_{t}=Q_{k+1}(x)=\One\{2n_{k+1}-mQ_{0}(x)>0\}$ by Proposition~\ref{prop:Q-simplified}.
A Bernstein inequality implies that on an event $\mathcal{E}^{(1,k)}$
of probability at least $1-\delta_{k}/2$, $U_{k}\leq2\tau_{k+1}\Pr(D_{k+1}\cap\{x:Q_{0}(x)<\frac{2n_{k+1}}{m}\})+2\log\frac{4}{\delta_{k}}$.

Define $\mathcal{E}^{(1)}:=\bigcap_{k=0}^{K-1}\mathcal{E}^{(1,k)}$,
and $\mathcal{E}^{(2)}:=\mathcal{E}^{(0)}\cap\mathcal{E}^{(1)}$.
By a union bound, we have $\Pr(\mathcal{E}^{(2)})\geq1-\delta$. Now,
on event $\mathcal{E}^{(2)}$, for any $k<K$, $D_{k+1}\subseteq\text{DIS}(B(h^{\star},2\nu+\epsilon_{k}))$,
so by Lemma~\ref{lem:dis-coefficient} $\Pr(D_{k+1}\cap\{x:Q_{0}(x)<\frac{2n_{k+1}}{m}\})\leq(2\nu+\epsilon_{k})\tilde{\theta}(2\nu+\epsilon_{k},\frac{2n_{k+1}}{m})$.
Therefore, the total number of label queries

\begin{align*}
\sum_{k=0}^{K-1}U_{k}\leq & \tau_{1}+\sum_{k=1}^{K-1}2\tau_{k+1}\Pr(D_{k+1}\cap\{x:Q_{0}(x)<\frac{2n_{k+1}}{m}\})+2K\log\frac{4}{\delta_{K}}\\
\leq & 1+2\sum_{k=1}^{K-1}\tau_{k+1}(2\nu+\epsilon_{k})\tilde{\theta}(2\nu+\epsilon_{k},\frac{2n_{k+1}}{m})+2K\log\frac{4}{\delta_{K}}\\
\leq & 1+2K\log\frac{4}{\delta_{K}}+2\tilde{\theta}(2\nu+\epsilon_{K-1},\frac{2n}{m})\cdot\Biggl(2n\nu\\
 & \left.+\gamma_{2}\sum_{k=1}^{K-1}(\frac{\tau_{k+1}M_{k}}{m+n_{k}}\log\frac{|\mathcal{H}|}{\delta_{k}}+\frac{\tau_{k+1}M_{k}^{2}}{(m+n_{k})^{\frac{3}{2}}}\sqrt{\log\frac{|\mathcal{H}|}{\delta_{k}}}+\tau_{k+1}\sqrt{\frac{M_{k}}{m+n_{k}}\nu\log\frac{|\mathcal{H}|}{\delta_{k}}})\right).
\end{align*}

Recall that $\alpha=\frac{m}{n}$,$\tau_{k}=2^{k}$, $\xi=\min_{1\leq k\leq K}\{M_{k}/\frac{m+n_{k}}{mq_{0}+n_{k}}\}$,
$\bar{M}=\max_{1\leq k\leq K}M_{k}$. We have $\sum_{k=1}^{K-1}\frac{\tau_{k+1}M_{k}}{m+n_{k}}\leq\sum_{k=1}^{K-1}\frac{\xi\tau_{k}}{mq_{0}+n_{k}}\leq\sum_{k=1}^{K}\frac{\xi n_{k}}{\alpha n_{k}q_{0}+n_{k}}\leq\frac{K\xi}{\alpha q_{0}+1}$
where the first inequality follows as $\frac{M_{k}}{m+n_{k}}\leq\frac{\xi}{mq_{0}+n_{k}}$,
and the second follows by $m=n\alpha\geq n_{k}\alpha$. Besides, $\sum_{k=1}^{K-1}\frac{\tau_{k}M_{k}^{2}}{(m+n_{k})^{\frac{3}{2}}}\leq\sum_{k=1}^{K-1}\frac{\tau_{k}M_{k}\xi}{\sqrt{m+n_{k}}(mq_{0}+n_{k})}\leq\sum_{k=1}^{K-1}\frac{\bar{M}\xi}{\sqrt{m+n_{k}}}\leq\frac{K\bar{M}\xi}{\sqrt{n\alpha}}$
where the first inequality follows as $\frac{M_{k}}{m+n_{k}}\leq\frac{\xi}{mq_{0}+n_{k}}$,
and the second follows as $M_{k}\leq\bar{M}$ and $\tau_{k}\leq mq_{0}+n_{k}$.
Finally, $\sum_{k=1}^{K}\tau_{k}\sqrt{\frac{M_{k}}{m+n_{k}}}\leq\sum_{k=1}^{K}\sqrt{\frac{\tau_{k}\xi}{\alpha q_{0}+1}}\leq\sqrt{\frac{n\xi}{\alpha q_{0}+1}}$
where the first inequality follows as $\frac{M_{k}}{m+n_{k}}\leq\frac{\xi}{mq_{0}+n_{k}}$
and $mq_{0}+n_{k}\geq\tau_{k}(\alpha q_{0}+1)$.

Therefore,

\begin{align*}
\sum_{k=0}^{K-1}U_{k}\leq & 1+2K\log\frac{4}{\delta_{K}}+2\tilde{\theta}(2\nu+\epsilon_{K-1},\frac{2n}{m})\Biggl(2n\nu\\
 & \left.+\gamma_{2}(\frac{K\xi}{\alpha q_{0}+1}\log\frac{K^{2}|\mathcal{H}|}{\delta}+\frac{K\bar{M}\xi}{\sqrt{n\alpha}}\sqrt{\log\frac{K^{2}|\mathcal{H}|}{\delta}}+\sqrt{\frac{n\xi\nu}{\alpha q_{0}+1}\log\frac{K^{2}|\mathcal{H}|}{\delta}})\right).
\end{align*}
\end{proof}

\section{\label{sec:app-vcis}Proofs and Examples for Sections~\ref{sec:Bias-Variance-Trade-Off}~and~\ref{sec:Active-Learning}}

\subsection*{Generalization Error Bound}

Theorem~\ref{thm:paper-smr-erm-gen} and Corollary~\ref{cor:paper-gen-err}
are immediate from the following theorem.
\begin{thm}
\label{thm:general-err-bnd}Let $\hat{h}_{M}=\arg\min_{h\in\calH}l(h;S,M)+\sqrt{\frac{\lambda}{m}\hvar(h;S,M)}$.
For any $\delta>0$, $M\geq1$, $\lambda\geq4\log\frac{|\mathcal{H}|}{\delta}$,
with probability at least $1-\delta$ over the choice of $S$, 
\begin{align}
l(\hat{h}_{M})-l(h^{\star})\leq & \frac{2\lambda M}{m}+\frac{16M}{3m}\log\frac{|\mathcal{H}|}{\delta}+\frac{M^{2}}{m^{\frac{3}{2}}}\sqrt{4\log\frac{|\mathcal{H}|}{\delta}}\label{eq:cvriw-h-hat-bound}\\
 & +\sqrt{\frac{\lambda}{m}\E\frac{\One\{h^{\star}(X)\neq Y\}}{Q_{0}(X)}\One[\frac{1}{Q_{0}(X)}\leq M]}+\Pr_{X}(\frac{1}{Q_{0}(X)}>M).\nonumber 
\end{align}
\end{thm}

\begin{proof}
The proof is similar to the proofs for Theorem~\ref{thm:Convergence}
and \ref{thm:gen}, and is omitted.
\end{proof}

\subsection*{Second Moment Regularizer}
\begin{proof}
(of Theorem~\ref{thm:paper-iw-gen-lb}) For any $0<\nu<\frac{1}{3}$,
$m>\frac{49}{\nu^{2}}$, set $q_{0}=\frac{1}{40}\nu$, $c=\frac{1}{3}$,
$\epsilon=\frac{c^{2}+\sqrt{c^{4}+4c^{2}q_{0}\nu m}}{2q_{0}m}$. It
can be checked that $\epsilon<\nu$ and $m=c^{2}\frac{\nu+\epsilon}{q_{0}\epsilon^{2}}$.
Let $\calX=\{x_{1},x_{2},x_{3}\}$, and define $\Pr(X=x_{1})=\nu$,
$\Pr(X=x_{2})=\nu+\epsilon$, $\Pr(X=x_{3})=1-2\nu-\epsilon$, and
$\Pr(Y=1)=1$. Let $\calH=\{h_{1},h_{2}\}$ where $h_{1}(x_{1})=-1$,
$h_{1}(x_{2})=h_{1}(x_{3})=1$, and $h_{2}(x_{2})=-1$, $h_{2}(x_{1})=h_{2}(x_{3})=1$.
Define the logging policy $Q_{0}(x_{1})=Q_{0}(x_{3})=1$, $Q_{0}(x_{2})=q_{0}$.
Let $S=\{(X_{t},Y_{t},Z_{t})\}_{t=1}^{m}$ be a dataset of size $m$
generated from the aforementioned distribution. Clearly, we have $l(h_{1})=\nu$
and $l(h_{2})=\nu+\epsilon$. We next prove that $\Pr(l(h_{1},S)>l(h_{2},S))\geq\frac{1}{100}$.
This implies that with probability at least $\frac{1}{100},$$h_{2}$
is the minimizer of the importance weighted loss $l(h,S)$, and its
population error $\Pr(h_{2}(X)\neq Y)=\nu+\epsilon=\nu+\frac{1}{q_{0}m}+\sqrt{\frac{\nu}{q_{0}m}}$.

We have 
\begin{align*}
\Pr(l(h_{1},S)>l(h_{2},S)) & \geq\Pr(l(h_{1},S)>\nu-\frac{\epsilon}{2}\text{ and }l(h_{2},S)<\nu-\frac{\epsilon}{2})\\
 & =1-\Pr(l(h_{1},S)\leq\nu-\frac{\epsilon}{2}\text{ or }l(h_{2},S)\geq\nu-\frac{\epsilon}{2})\\
 & \geq1-\Pr(l(h_{1},S)\leq\nu-\frac{\epsilon}{2})-\Pr(l(h_{2},S)\geq\nu-\frac{\epsilon}{2})\\
 & =\Pr(l(h_{2},S)<\nu-\frac{\epsilon}{2})-\Pr(l(h_{1},S)\leq\nu-\frac{\epsilon}{2})
\end{align*}

Observe that by our construction, $ml(h_{1},S)=\sum_{i=1}^{m}\One\{X_{i}=x_{1}\}$
follows the binomial distribution $\Bin(m,\nu)$. By a Chernoff bound,
$\Pr(l(h_{1},S)\leq\nu-\frac{\epsilon}{2})\leq e^{-\frac{1}{2}m\epsilon^{2}}$.
Since $\epsilon\geq\sqrt{\frac{c^{2}\nu}{q_{0}m}}\geq\sqrt{\frac{40c^{2}}{m}}$,
$e^{-\frac{1}{2}m\epsilon^{2}}\leq e^{-20c^{2}}=e^{-\frac{20}{9}}$.

By our construction, we also have that $q_{0}ml(h_{2},S)=\sum_{i=1}^{m}\One\{X_{i}=x_{2},Z_{i}=1\}$
which follows the binomial distribution $\Bin(m,q_{0}(\nu+\epsilon))$.
Thus, $\Pr(l(h_{2},S)\leq\nu-\frac{\epsilon}{2})=\Pr(q_{0}ml(h_{2},S)\leq q_{0}m(\nu+\epsilon)-\frac{3}{2}q_{0}m\epsilon)\geq\frac{1}{\sqrt{2\pi}}\frac{3c}{9c^{2}+1}e^{-\frac{9}{2}c^{2}}=\frac{1}{2\sqrt{2\pi}}e^{-\frac{1}{2}}$
where the inequality follows by Lemma~\ref{lem:bin-lb}.

Therefore, $\Pr(l(h_{1},S)>l(h_{2},S))\geq\Pr(l(h_{2},S)<\nu-\frac{\epsilon}{2})-\Pr(l(h_{1},S)\leq\nu-\frac{\epsilon}{2})\geq\frac{1}{2\sqrt{2\pi}}e^{-\frac{1}{2}}-e^{-\frac{20}{9}}\geq\frac{1}{100}$.
\end{proof}
\begin{rem}
A similar result for general cost-sensitive empirical risk minimization
is proved in \cite{MP09,ND17}. In \cite{MP09,ND17}, they construct
examples where $\text{Var}(h^{\star})=0$ and learning $h^{\star}$
with unregularized ERM gives $\tilde{\Omega}(\sqrt{\frac{1}{m}})$
error, while regularized ERM gives $\tilde{O}(\frac{1}{m})$ error.
However, their construction does not work in our setting because the
bound for unregularized ERM \cite{YCJ18} also gives $\tilde{O}(\frac{1}{m})$
error when $\text{Var}(h^{\star})=0$ (since $\text{Var}(h^{\star})=0$
implies $l(h^{\star})=0$), so more careful construction and analysis
are needed.
\end{rem}

\subsection*{Clipping}

The clipping threshold $M_{0}$ is chosen to minimize an error bound
for the clipped second-moment regularized ERM. According to Theorem~\ref{thm:general-err-bnd},
we would like to choose $M$ that minimizes the RHS of (\ref{eq:cvriw-h-hat-bound}).
We set $\lambda=4\log\frac{|\mathcal{H}|}{\delta}$ in Theorem~\ref{thm:general-err-bnd},
focus on the low order terms with respect to $m$, and minimize $e(M):=\sqrt{\frac{4\log\frac{|\mathcal{H}|}{\delta}}{m}\E\frac{1}{Q_{0}(X)}\One[\frac{1}{Q_{0}(X)}\leq M]}+\Pr_{X}(\frac{1}{Q_{0}(X)}>M)$
instead since $\One\{h^{\star}(X)\neq Y\}$ could not be determined
with unlabeled samples. In this sense, the following proposition shows
that our choice of $M$ is nearly optimal.
\begin{prop}
\label{prop:opt-c}Suppose random variable $\frac{1}{Q_{0}(X)}$ has
a probability density function, and there exists $M_{0}\geq1$ such
that $\frac{2\log\frac{|\mathcal{H}|}{\delta}}{m}M_{0}=\Pr_{X}(\frac{1}{Q_{0}(X)}>M_{0})$.
Then $e(M_{0})\leq\sqrt{2}\inf_{M\geq1}e(M).$
\end{prop}

\begin{proof}
Define $f_{1}(M)=\frac{4\log\frac{|\mathcal{H}|}{\delta}}{m}\E\frac{1}{Q_{0}(X)}\One[\frac{1}{Q_{0}(X)}\leq M]$,
and $f_{2}(M)=\Pr_{X}(\frac{1}{Q_{0}(X)}>c)$. We first show that
$f_{1}(M_{0})+f_{2}(M_{0})^{2}\leq\inf_{M>1}f_{1}(M)+f_{2}(M)^{2}$.

Let $g(x)$ be the probability density function of random variable
$1/Q_{0}(X)$. We have $f_{1}(M)=\frac{4\log\frac{|\mathcal{H}|}{\delta}}{m}\int_{0}^{M}xg(x)\d x$
and $f_{2}(M)=\int_{M}^{\infty}g(x)\d x$, so $f_{1}'(M)=\frac{4\log\frac{|\mathcal{H}|}{\delta}}{m}Mg(M)$,
and $f_{2}'(M)=-g(M)$. Define $f(M)=f_{1}(M)+f_{2}(M)^{2}$. We have
\begin{align*}
f'(M) & =f_{1}'(M)+2f_{2}'(M)f_{2}(M)\\
 & =2g(M)(\frac{2\log\frac{|\mathcal{H}|}{\delta}}{m}M-f_{2}(M)).
\end{align*}

Recall we assume there exists $M_{0}\geq1$ such that $\frac{2\log\frac{|\mathcal{H}|}{\delta}}{m}M_{0}=f_{2}(M_{0})$.
Since $\frac{2\log\frac{|\mathcal{H}|}{\delta}}{m}M$ is strictly
increasing w.r.t.~$M$ and $f_{2}(M)$ is non-increasing w.r.t. $M$,
it follows that $f(M)$ achieves its minimum at $M_{0}$, that is,
for any $c\geq1$, $f_{1}(M_{0})+f_{2}^{2}(M_{0})\leq f_{1}(M)+f_{2}^{2}(M)$.

Now, $\sqrt{f_{1}(M_{0})+f_{2}^{2}(M_{0})}\geq\frac{1}{\sqrt{2}}(\sqrt{f_{1}(M_{0})}+f_{2}(M_{0}))$
since $\sqrt{a+b}\geq\frac{1}{\sqrt{2}}(\sqrt{a}+\sqrt{b})$ for any
$a,b\geq0$, and $\sqrt{f_{1}(M)+f_{2}^{2}(M)}\leq\sqrt{f_{1}(M)}+f_{2}(M)$
since $\sqrt{a+b}\leq\sqrt{a}+\sqrt{b}$ for any $a,b\geq0$. Thus
$\frac{1}{\sqrt{2}}(\sqrt{f_{1}(M_{0})}+f_{2}(M_{0}))\leq\sqrt{f_{1}(M)}+f_{2}(M)$
for all $M>0$, which concludes the proof.
\end{proof}
\begin{rem}
Since $\frac{1}{M}\Pr_{X}(\frac{1}{Q_{0}(X)}>M)$ is monotonically
decreasing with respect to $M$ and its range is $(0,1)$, the existence
and uniqueness of $M_{0}$ are guaranteed if $\frac{2}{m}\log\frac{|\mathcal{H}|}{\delta}<1$.
\end{rem}

The following example shows that our choice of $M$ indeed avoids
outputting suboptimal classifiers.
\begin{example}
\label{exa:clipping}Let $\calX=\{x_{0},x_{1},x_{2},x_{3},x_{4}\}$,
$\calH=\{h_{1},h_{2},h_{3},h_{4}\}$. Suppose $\Pr(Y=1)=-1$, $\nu<\frac{1}{10}$,
$\alpha<0.01$, and $\epsilon=\frac{\nu}{1+1/100\alpha}$. The marginal
distribution on $X$, the prediction of each classifier, and the logging
policy $Q_{0}$ is defined in Table~\ref{tab:eg-clipping}.

\begin{table}[h]
\caption{\label{tab:eg-clipping}An example for clipping}

\centering{}%
\begin{tabular}{cccccc}
\toprule 
 & $x_{0}$ & $x_{1}$ & $x_{2}$ & $x_{3}$ & $x_{4}$\tabularnewline
\midrule
\midrule 
$h_{1}(\cdot)$ & 1 & 1 & -1 & -1 & -1\tabularnewline
\midrule 
$h_{2}(\cdot)$ & 1 & -1 & 1 & -1 & -1\tabularnewline
\midrule 
$h_{3}(\cdot)$ & 1 & -1 & -1 & 1 & -1\tabularnewline
\midrule 
$h_{4}(\cdot)$ & -1 & -1 & -1 & -1 & 1\tabularnewline
\midrule 
$\Pr_{X}(\cdot)$ & $\nu-\epsilon$ & $\epsilon$ & $4\epsilon$ & $16\epsilon$ & $1-\nu-20\epsilon$\tabularnewline
\midrule 
$Q_{0}(\cdot)$ & 1 & $\alpha$ & $\alpha$ & $4\alpha$ & $4\alpha$\tabularnewline
\bottomrule
\end{tabular}
\end{table}

We have $l(h_{1})=\nu$, $l(h_{2})=\nu+3\epsilon$, $l(h_{3})=\nu+15\epsilon$,
$l(h_{4})=1-\nu-20\epsilon$. Next, we consider when examples with
$Q_{0}$ equals $\alpha$, i.e. examples on $x_{1}$ and $x_{2}$,
should be clipped. We set the failure probability $\delta=0.01$.

If $m\geq\frac{28}{\alpha\epsilon},$ without clipping our error bound
guarantees that (by minimizing a regularized training error) learner
can achieve an error of less than $\nu+3\epsilon$, so it would output
the optimal classifier $h_{1}$ with high probability. On the other
hand, if $M<\frac{1}{\alpha}$, then all examples on $x_{1}$ and
$x_{2}$ are ignored due to clipping, so the learner would not be
able to distinguish between $h_{1}$and $h_{2}$, and thus with constant
probability the error of the output classifier is at least $l(h_{2})=\nu+3\epsilon$.
This means if $m\geq\frac{28}{\alpha\epsilon}$, examples on $x_{1}$
and $x_{2}$ should not be clipped.

If $m\geq\frac{2}{\alpha\epsilon}$ and examples on $x_{1}$ and $x_{2}$
are clipped, our error bound guarantees learner can achieve an error
of less than $\nu+16\epsilon$, which means the learner would output
either $h_{1}$ or $h_{2}$ and achieve an actual error of at most
$\nu+3\epsilon$. However, without clipping, the learner would require
$m\geq\frac{4}{\alpha\epsilon}$ to achieve an error of less than
$\nu+16\epsilon$. Thus, if $m\leq\frac{4}{\alpha\epsilon}$, examples
on $x_{1}$ and $x_{2}$ should be clipped.

To sum up, examples with $Q_{0}$ equals $\alpha$ (i.e. $x_{1}$
and $x_{2}$) should be clipped if $m\leq\frac{4}{\alpha\epsilon}$
and not be clipped if $m\geq\frac{28}{\alpha\epsilon}$. Our choice
of the clipping threshold clips $x_{1}$ and $x_{2}$ whenever $m\leq\frac{24}{5\alpha\epsilon}$,
which falls inside the desired interval.
\end{example}

\subsection*{\DEBIAS Strategy}

The following example shows the \debias strategy indeed improves
label complexity.
\begin{example}
\label{exa:debias}Let $\lambda>1$ be any constant. Suppose $\calX=\{x_{1},x_{2}\}$,
$Q_{0}(x_{1})=1$, $Q_{0}(x_{2})=\alpha$, $\Pr(x_{1})=1-\mu$, $\Pr(x_{2})=\mu$
and assume $\mu\leq\frac{1}{4\lambda}$ and $\alpha\leq\frac{\mu^{2}}{2\lambda}$.
Assume the logged data size $m$ is greater than twice as the online
stream size $n$. Without the \debias strategy, after seeing $n$
examples, the learner queries all $n$ examples and achieves an error
bound of $\frac{4\log\frac{2|{\cal H}|}{\delta}}{3(m\alpha+n)}+\sqrt{4(\frac{c\mu}{m+n}+\frac{\mu}{m\alpha+n})\log\frac{2|{\cal H}|}{\delta}}$
by minimizing the regularized MIS loss. With the sample selection
bias correction strategy, the learner only queries $x_{2}$, so after
seeing $n$ examples, it queries only $\mu n$ examples in expectation
and achieves an error bound of $\frac{4\log\frac{2|{\cal H}|}{\delta}}{3(m\alpha+n)}+\sqrt{4(\frac{c\mu}{m}+\frac{\mu}{m\alpha+n})\log\frac{2|{\cal H}|}{\delta}}$.
With some algebra, it can be shown that to achieve the same error
bound, if $\frac{\lambda\alpha}{\mu}m\leq n\leq\frac{\mu}{2}m$, then
the number of queries requested by the learner without the \debias
correction strategy is at least $\lambda$ times more than the number
of queries for the learner with the bias correction strategy. Since
this holds for any $\lambda\geq1$, the decrease of the number of
label queries due to our \debias strategy can be significant.
\end{example}

\section{Experiments\label{sec:exp}}
We conduct experiments to compare the performance 
of the proposed active learning algorithm against some baseline methods. Our 
experiment results confirm our theoretical analysis that the test error of the
proposed algorithm drops faster than alternative methods as the number of label
queries increases.

\subsection{Methodology}

\subsubsection*{Algorithms and Implementations}
We consider the following algorithms:
\begin{itemize}
    \item \algoPassive: A passive learning algorithm that queries labels for 
    all examples. It directly optimizes an importance weighted estimator;
    \item \algoActive: The active learning algorithm proposed in~\cite{YCJ18}.
    It applies the disagreement-based active learning framework, multiple 
    importance sampling, and a \debias strategy.
    \item \algoVCActive: Algorithm~\ref{alg:main} proposed in this paper. It 
    applies the \emph{variance controlled} disagreement-based active learning
    framework, multiple importance sampling, and an \emph{improved} \debias 
    strategy.
\end{itemize}

Similar to~\cite{YCJ18}, our implementation of disagreement-based active learning
framework follows the Vowpal Wabbit (\cite{vw}) package. In particular,
\begin{itemize}
    \item We set the hypothesis space to be the set of linear classifiers, and 
    replace the 0-1 loss with a squared loss.
    \item We do not explicitly maintain the candidate set $C_k$ or the disagreement
    region $D_k$. To compute $\hat{h}_k$ in line 6 of Algorithm~\ref{alg:main}, we
    ignore the constraint $h\in C_k$ and conduct online gradient descent with step 
    size $\sqrt{\frac{\eta}{\eta+t}}$. To approximately
    check whether $x\in D_{k+1}$ in line 15, let $w_k$ be the normal vector 
    for $\hat{h}_k$, and $a$ be current step size. We claim $x\in D_{k+1}$ if 
    $\frac{|2w_k^\top x|}{a x^\top x} \leq 
    \sqrt{\frac{C\cdot \hat{V}(\hat{h}_k; \tilde{S}_k, M_k)}{m+n_k}} 
    + \frac{C\cdot M_k}{m+n_k}$. 
    Here $C$ is a parameter that captures the model capacity 
    (this corresponds to the $\log\frac{|\calH|}{\delta}$ term in the error bound; 
    as noted in \cite{H10}, this is often loose and needs to be tuned as a parameter in practice)
    and we tune this parameter in experiments.

\end{itemize}

Besides, we incorporate \vcis into active learning through the following way:
\begin{itemize}
    \item In order to find the clipping threshold $M_k$ (line~\ref{alg:Mk} of 
    Algorithm~\ref{alg:main}), we empirically estimate 
    $\Pr(\frac{m+n_k}{mQ_0(X)+n_k}> M/2)$ on the logged observational data (note that 
    this estimation does not involve labels).
    \item We follow~\cite{SJ15CRM} to approximately calculate the online gradient for 
    optimization with a variance regularizer.
\end{itemize}

\subsubsection*{Data}
We generate a synthetic dataset where 6000 examples are drawn uniformly at random 
from $[0,1]^{30}$, and labels are assigned by a linear separator and get flipped with 
probability 0.05. We randomly split the dataset into 80\% training data and 20\% test 
data. Among the training dataset, we randomly choose around 50\% as logged observational
data, and apply a synthetic logging policy to choose which labels in the observational 
data set are revealed to the algorithm. Our experiments use the following two policies:

\begin{itemize}
    \item Certainty: We first find a linear hyperplane that approximately separates the
     data. Then, we reveal the label with a higher probability (i.e., larger $Q_0$ 
     value) if the example is further away from this hyperplane.
    \item Uncertainty: We first find a linear hyperplane that approximately separates the
    data. Then, we reveal the label with a higher probability (i.e., larger $Q_0$ 
    value) if the example is closer to this hyperplane.
\end{itemize}

\subsubsection*{Parameter Tuning}
We follow \cite{HAHLS15} and \cite{YCJ18} to tune the model capacity $C$ and learning rate 
$\eta$, and report the best result for each algorithm under each logging policy.

In particular, let $e(i,A,p,l)$ be the test error of algorithm
$A$ with parameter set $p=(C,\eta)$ after making $l$ label queries during
the $i$-th trial ($i=1,2,\dots,N$).
We evaluate the performance of the algorithm $A$ with parameter set $p$ by 
following Area Under the error-label Curve metric: $\text{AUC}(A,p) = \frac1{2N}\sum_{i=1}^{N}\sum_l (e(i,A,p,l+1)+e(i,A,p,l))$.
At the end, for each algorithm, we report the error-label curve achieved with the parameter 
set $p$ that minimizes $\text{AUC}(A,p)$.

In our experiments, we try $C$ in $\{0.01 \times 2^i \mid i=0,2,4,\cdots,10\}$, 
and $\eta$ in $\{0.0001 \times 2^i \mid i=0,2,4,\cdots,12\}$. For each algorithm, 
policy, and parameter set, the experiments are repeated for $N=16$ times.

\subsection{Results}
\begin{figure}
    \begin{centering}
    \subfloat[Certainty]{
        \begin{centering}
        \includegraphics[width=0.5\textwidth]{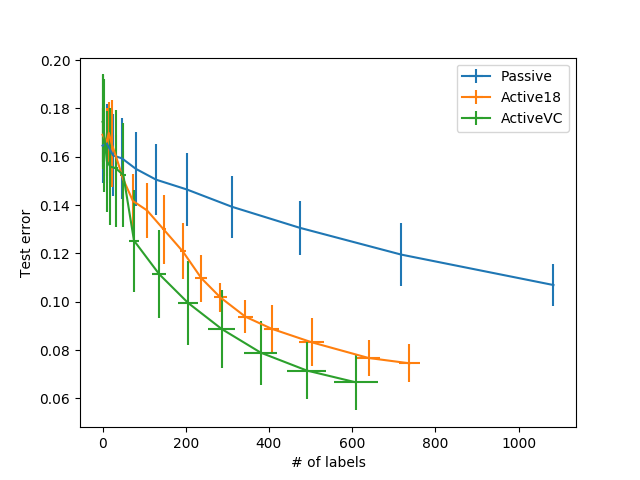}
        \par\end{centering}
    }
    \subfloat[Uncertainty]{
        \begin{centering}
        \includegraphics[width=0.5\textwidth]{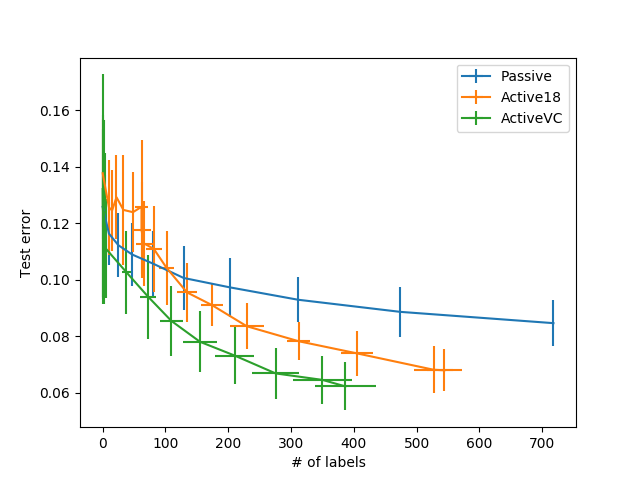}
        \par\end{centering}
    }
    \caption{\label{fig:error-vs-labels}Test error vs the number of labels under
    different logging policies with the best parameters.}
    \end{centering}
\end{figure}
    
We plot test error as a function of the number of labels in Figure~\ref{fig:error-vs-labels}.
It shows that test errors achieved by the proposed method drop faster than both the passive
learning baseline, and the prior work~\cite{YCJ18} which does not apply variance control techniques.
Additionally, as the number of labels grows, the gap widens.

\end{document}